\documentclass[runningheads]{llncs}
\usepackage[utf8]{inputenc} 
\usepackage[T1]{fontenc}    
\usepackage[breaklinks]{hyperref}
\hypersetup{
colorlinks = true, 
linkcolor=blue, 
citecolor=blue, 
urlcolor=blue 
}
\usepackage{url}            
\usepackage{booktabs}       
\usepackage{amsfonts}       
\usepackage{nicefrac}       
\usepackage{microtype}      
\usepackage{xcolor}         

\usepackage{graphicx}
\usepackage{multirow}
\usepackage{amsmath}
\usepackage{xparse}
\usepackage{xspace}
\usepackage{subcaption}
\usepackage{paralist}
\usepackage{wrapfig}
\usepackage[linesnumbered,ruled,vlined]{algorithm2e}
\usepackage{comment}
\usepackage{orcidlink}
\usepackage[misc]{ifsym}

\newcommand{\R}{\ensuremath{\mathbb{R}}\xspace}
\newcommand{\true}{\ensuremath{\mathsf{true}}\xspace}
\newcommand{\false}{\ensuremath{\mathsf{false}}\xspace}
\newcommand{\transformerSymb}{\ensuremath{f}\xspace}
\newcommand{\transformer}{\ensuremath{\transformerSymb}\xspace}
\newcommand{\transformerSubOne}{\ensuremath{\transformerSymb_{\textsc{Sub1}}}\xspace}
\newcommand{\transformerSubTwo}{\ensuremath{\transformerSymb_{\textsc{Sub2}}}\xspace}
\newcommand{\modelLayer}{\ensuremath{N}\xspace}
\newcommand{\highX}{\ensuremath{X}\xspace}
\newcommand{\highXEle}[1]{\ensuremath{X_{\mathit{#1}}}\xspace}
\newcommand{\highY}{\ensuremath{Y}\xspace}
\newcommand{\highYEle}[1]{\ensuremath{Y_{{#1}}}\xspace}
\newcommand{\highXPerturb}
{\ensuremath{X'}\xspace}
\newcommand{\highXPerturbEle}[1]
{\ensuremath{X'_{\mathit{#1}}}\xspace}
\newcommand{\perturbIdx}{\ensuremath{d}\xspace}
\newcommand{\perturbIdxSet}{\ensuremath{D}\xspace}

\newcommand{\fattention}{\ensuremath{f_{\textsc{Attn}}}\xspace}
\newcommand{\faddnormone}{\ensuremath{f_{\mathrm{AN1}}}\xspace}
\newcommand{\faddnormtwo}{\ensuremath{f_{\mathrm{AN2}}}\xspace}
\newcommand{\ffnn}{\ensuremath{f_{\textsc{FNN}}}\xspace}
\newcommand{\fpool}{\ensuremath{f_{\textsc{Pool}}}\xspace}
\newcommand{\Qmat}{\ensuremath{Q}\xspace}
\newcommand{\Kmat}{\ensuremath{K}\xspace}
\newcommand{\Vmat}{\ensuremath{V}\xspace}
\newcommand{\QWmatrix}{\ensuremath{W_Q}\xspace}
\newcommand{\KWmatrix}{\ensuremath{W_K}\xspace}
\newcommand{\VWmatrix}{\ensuremath{W_V}\xspace}
\newcommand{\transpose}[1]{\ensuremath{#1^\mathrm{T}}\xspace}
\newcommand{\softmax}{\ensuremath{\mathrm{softmax}}\xspace}
\newcommand{\relu}{\ensuremath{\mathrm{ReLU}}\xspace}
\newcommand{\numClass}{\ensuremath{c}\xspace}
\newcommand{\numInputFirst}{\ensuremath{n}\xspace}
\newcommand{\numInputSecond}{\ensuremath{m}\xspace}

\newcommand{\inSpec}{\ensuremath{\Phi}\xspace}
\newcommand{\outSpec}{\ensuremath{\Psi}\xspace}
\newcommand{\Defeq}{\ensuremath{:=}\xspace}

\newcommand{\alphaU}{\ensuremath{\alpha^\mathrm{U}}\xspace}
\newcommand{\alphaL}{\ensuremath{\alpha^\mathrm{L}}\xspace}
\newcommand{\betaU}{\ensuremath{\beta^\mathrm{U}}\xspace}
\newcommand{\betaL}{\ensuremath{\beta^\mathrm{L}}\xspace}
\newcommand{\gammaU}{\ensuremath{\gamma^\mathrm{U}}\xspace}
\newcommand{\gammaL}{\ensuremath{\gamma^\mathrm{L}}\xspace}
\newcommand{\fqk}{\ensuremath{f_{\mathsf{qk}}}\xspace}
\newcommand{\qkUpperOne}[1]{\ensuremath{\fqk^{\mathrm{U_1}}(#1)}\xspace}
\newcommand{\qkLowerOne}[1]{\ensuremath{\fqk^{\mathrm{L_1}}(#1)}\xspace}
\newcommand{\qkUpperTwo}[1]{\ensuremath{\fqk^{\mathrm{U_2}}(#1)}\xspace}
\newcommand{\qkLowerTwo}[1]{\ensuremath{\fqk^{\mathrm{L_2}}(#1)}\xspace}
\newcommand{\qkUpperNonlin}[1]{\ensuremath{\fqk^{\mathrm{U}}(#1)}\xspace}
\newcommand{\qkLowerNonlin}[1]{\ensuremath{\fqk^{\mathrm{L}}(#1)}\xspace}
\newcommand{\qkUpperLin}[1]{\ensuremath{\fqk^{\mathrm{U^*}}(#1)}\xspace}
\newcommand{\qkLowerLin}[1]{\ensuremath{\fqk^{\mathrm{L^*}}(#1)}\xspace}
\newcommand{\qUpper}{\ensuremath{q^{\mathrm{U}}}\xspace}
\newcommand{\qLower}{\ensuremath{q^{\mathrm{L}}}\xspace}
\newcommand{\kUpper}{\ensuremath{k^{\mathrm{U}}}\xspace}
\newcommand{\kLower}{\ensuremath{k^{\mathrm{L}}}\xspace}
\newcommand{\qEle}{\ensuremath{\Qmat_{\mathit{ih}}}\xspace}
\newcommand{\kEle}{\ensuremath{\Kmat_{\mathit{jh}}}\xspace}
\newcommand{\qEleU}{\ensuremath{\Qmat_{\mathit{ih}}^{\mathrm{U}}}\xspace}
\newcommand{\kEleU}{\ensuremath{\Kmat_{\mathit{jh}}^{\mathrm{U}}}\xspace}
\newcommand{\qEleL}{\ensuremath{\Qmat_{\mathit{ih}}^{\mathrm{L}}}\xspace}
\newcommand{\kEleL}{\ensuremath{\Kmat_{\mathit{jh}}^{\mathrm{L}}}\xspace}
\newcommand{\frelulow}{\ensuremath{f_{\mathrm{relu}}^{\mathrm{L}}}\xspace}

\newcommand{\reluInLow}{\ensuremath{\mathit{l}}\xspace}
\newcommand{\reluInUpp}{\ensuremath{\mathit{u}}\xspace}
\newcommand{\margin}{\ensuremath{\mathsf{Margin}}\xspace}
\newcommand{\loss}{\ensuremath{\mathcal{L}}\xspace}
\newcommand{\marginLow}{\ensuremath{\underline{\margin}}\xspace}

\DeclareDocumentCommand{\norm}{o m}{%
  \ensuremath{\left\lVert #2 \right\rVert_{#1}}%
}
\newcommand{\yelp}{\texttt{Yelp}\xspace}
\newcommand{\sst}{\texttt{SST}\xspace}
\newcommand{\crownBAF}{\texttt{CrownBaF}\xspace}

\newcommand{\globalUpper}[2]{\ensuremath{\transformerSymb^\mathrm{U}_{#1,#2}}\xspace}
\newcommand{\globalLower}[2]{\ensuremath{\transformerSymb^\mathrm{L}_{#1,#2}}\xspace}
\newcommand{\neuron}[3]{\ensuremath{\Phi^{(#1)}_{#2,#3}(X)}\xspace}
\newcommand{\neuronvec}[2]{\ensuremath{\Phi^{(#1)}_{#2,:}(X)}\xspace}
\newcommand{\neuronw}[1]{\ensuremath{\Theta^{(#1)}}\xspace}
\newcommand{\neuronb}[2]{\ensuremath{\Delta^{(#1)}_{#2,:}}\xspace}
\newcommand{\ablinearw}[2]{\ensuremath{w^{(#1,#2)}}\xspace}
\newcommand{\ablinearb}[3]{\ensuremath{b^{(#1,#2)}_{#3,:}}\xspace}
\newcommand{\relualLower}[4]{\ensuremath{w^{(#1,#2),L}_{#3,#4}}\xspace}
\newcommand{\relubeLower}[4]{\ensuremath{b^{(#1,#2),L}_{#3,#4}}\xspace}
\newcommand{\relualUpper}[4]{\ensuremath{w^{(#1,#2),U}_{#3,#4}}\xspace}
\newcommand{\relubeUpper}[4]{\ensuremath{b^{(#1,#2),U}_{#3,#4}}\xspace}
\newcommand{\qkvweight}[1]{\ensuremath{\Theta^{#1}}\xspace}
\newcommand{\qkvbias}[1]{\ensuremath{\Delta^{#1}}\xspace}
\newcommand{\qkvweightEle}[3]{\ensuremath{\Theta^{#1}_{#2,#3}}\xspace}
\newcommand{\qkvbiasEle}[3]{\ensuremath{\Delta^{#1}_{#2,#3}}\xspace}

\newcommand{\tool}{\ensuremath{\mathsf{BuFFeT}}\xspace}
\newcommand{\toolp}{r-\tool}
\newcommand{\toola}{o-\tool}

\newcounter{researchquestionCount}
\newcommand{\researchquestion}[1]{\stepcounter{researchquestionCount}\vspace{5pt}\noindent\parbox{0.97\textwidth}{{\bf RQ\arabic{researchquestionCount}} {\it #1}}\vspace{1pt}}

\makeatletter
\def\orcidID#1{\kern .08em\href{https://orcid.org/#1}{\includegraphics[keepaspectratio,width=0.9em]{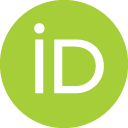}}}
\makeatother

\begin{document}

\title{Precise Verification of Transformers through ReLU-Catalyzed Abstraction Refinement}
%
%
\author{Hengjie Liu\inst{1}\orcidlink{0009-0005-5820-1168} \and
Zhenya Zhang\inst{1,2}$^{\text{(\Letter)}}$\orcidlink{0000-0002-3854-9846}  \and
Jianjun Zhao\inst{1}\orcidlink{0000-0001-8083-4352}}

\authorrunning{H. Liu et al.}

%

\institute{Kyushu University, Fukuoka, Japan \\ \email{liu.hengjie.142@s.kyushu-u.ac.jp},
\email{\{zhang, zhao\}@ait.kyushu-u.ac.jp} \and National Institute of Informatics, Tokyo, Japan}
\maketitle              

\begin{abstract}
Formal verification of transformers has become increasingly important due to their widespread deployment in safety-critical applications. Compared to classic neural networks, the inferences of transformers involve highly complex computations, such as dot products in self-attention layers, rendering their verification extremely difficult. Existing approaches explored over-approximation methods by constructing convex constraints to bound the output ranges of transformers, which can achieve high efficiency. However, they may sacrifice verification precision, and consequently introduce significant approximation error that leads to frequent occurrences of false alarms. 
In this paper, we propose a transformer verification approach that can achieve improved precision. At the core of our approach is a novel usage of ReLU, by which we represent a precise but non-linear bound for dot products such that we can further exploit the rich body of literature for convex relaxation of ReLU to derive precise bounds. We extend two classic approaches to the context of transformers, a rule-based one and an optimization-based one, resulting in two new frameworks for efficient and precise verification.
We evaluate our approaches on different model architectures and robustness properties derived from two datasets about sentiment analysis, and compare with the state-of-the-art baseline approach. Compared to the baseline, our approach can achieve significant precision improvement for most of the verification tasks with acceptable compromise of efficiency, which demonstrates the effectiveness of our approach. 
\end{abstract}
\section{Introduction}
\label{sec:introduction}
Over the past decade, the rapid development of deep learning has driven unprecedented breakthroughs in different application domains. As a ground-breaking milestone, transformers~\cite{NIPS2017_3f5ee243}, empowered by the \emph{self-attention} mechanism, have brought revolutionary performance advances to multiple fields, such as computer vision and natural language processing (NLP). However, similar to other types of neural networks, transformers are notoriously vulnerable to adversarial perturbations, which brings significant doubts on their trustworthiness. For instance, \emph{synonym substitution}~\cite{samanta2017craftingtextadversarialsamples,alzantot-etal-2018-generating} has proved to be an effective attack method in NLP that can fool the model by replacing some words of an input sentence with their synonyms. Given the increasing deployment of transformers in safety-critical applications (e.g., autonomous driving, healthcare), ensuring their robustness in noisy and sophisticated environments has become an urgent demand.

Verification is a preferred approach that can formally prove the safety and robustness of systems, thereby alleviating the concerns about their misbehaviors. In the context of neural networks, it has been extensively studied (see a survey~\cite{liu2021algorithms} and the competition report~\cite{brix2023first}) with well-established methodological foundations. The challenge mainly involves how to deal with their \emph{non-linear components}, e.g., activation functions in classic feed-forward neural networks. Due to this, computing an exact bound of neural network outputs often incurs exponential time complexity and thus is not scalable to real-world models. By contrast, \emph{abstract interpretation} is an approach that constructs convex \emph{over-approximation} of their outputs and assesses whether the approximated outputs satisfy the desired property---as long as the over-approximated outputs satisfy the property, the original outputs must also satisfy the property, thereby verifying the neural networks. Owning to the convexity of the approximated output bounds, these approaches are much more efficient and thus preferable in practice.

\smallskip\noindent\textit{Overview of Related Work} Despite the prosperity of research in classic neural network verification, verification of \emph{transformers} still remains at a very early stage. Nevertheless, standing on the shoulders of classic verification, the research community get straight to abstract interpretation~\cite{Shi2020Robustness,bonaert2021fast,zhang2024galileo}, which is more scalable to transformers of large sizes. For example, as the state-of-the-art, Shi et al.~\cite{Shi2020Robustness} extend the idea of linear relaxation from classic neural network verification, and devise different linear bounds for different non-linear components of  transformers, such as the dot product operations in self-attention layers. Compared to the verification approaches that bound transformer outputs exactly (e.g., \cite{liao2023transformers}),  computation of convex over-approximation often requires polynomial time complexity, so they are much faster and more scalable, and therefore they are feasible to be used to handle complex transformers in real world.

\smallskip\noindent\textit{Completeness Issues of Abstract Interpretation} As an intrinsic issue of abstract interpretation, if the over-approximated outputs violate the given property, it does not imply that the original outputs also violate the property, and in this case, the reported violation (i.e., the violation by the over-approximated outputs) could be a \emph{false alarm}. This is known as the \emph{completeness issue}~\cite{liu2021algorithms}, and essentially, it arises from the error introduced during the approximation of the transformers' non-linear components. In particular, we show that existing over-approximation approaches are often \emph{suboptimal} and thus may introduce significant approximation error and bring many false alarms. Given that verification is typically expensive, frequent occurrences of false alarms can lead to substantial waste of time and resources, highlighting the need for more precise approximation methods.

\smallskip\noindent\textit{Key Idea of Abstraction Refinement} While there are various non-linear operations in transformers, dot products happen most frequently in matrix multiplications of self-attention layers and pose unique challenges that do not exist in classic neural network verification. In~\cite{Shi2020Robustness}, the main contribution involves the design of a provably correct planar bound for dot products, which is given in terms of a rigid strategy that always covers one side of the hyperbolic paraboloid derived from the dot product (see Fig.~\ref{fig:surface2} for illustration).
While it is a sound choice, it is not the unique one, e.g., there can be a dual bound that covers another side of the hyperbolic paraboloid (see Fig.~\ref{fig:surface3} for illustration).

However, none of the two planar bounds, including the one in~\cite{Shi2020Robustness} and the dual bound, is the optimal under all different situations. To refine the approximation, we need to devise new approaches that fuse the two planar bounds, meanwhile holding the linearity of the bounds to maintain verification efficiency.

\smallskip\noindent\textit{Contributions}
In this paper, we explore a novel usage of ReLU functions, by which we propose \tool, standing for \textbf{B}o\textbf{U}nd \textbf{F}using-based re\textbf{F}in\textbf{E}ment for \textbf{T}ransformers, which fuses the planar bounds of dot product operations and represents them with the assistance of ReLU functions. The advantage of doing this is that, we can then leverage the relaxed linear bounds of the ReLU-based representations, by exploiting the rich body of existing literature in classic neural network verification, in which linear relaxation of ReLUs has been extensively studied and a variety of bound-tightening techniques have been developed. Consequently, we can derive tighter bounds for dot products---the fundamental operations in self-attention layers---which in turn leads to tighter bounds on transformer outputs. Moreover, based on our representation, we also show that the bounds proposed in~\cite{Shi2020Robustness} and its dual case, are two specific instantiations that can be derived by our representation, and we explain why they are suboptimal. 

Furthermore, we leverage the insights from two existing strategies from classic neural network verification to derive improved approximation. The first strategy, adapted from~\cite{singh2019abstract}, is simple but effective, which, based on the input ranges of ReLUs, selects the optimal planar bound to refine the approximation. The second strategy, inspired by~\cite{xu2020fast}, is more complex---it formalizes the refinement of approximation as an optimization problem and searches in the space of ReLUs' bounds to find the optimal solution. Thanks to this formalization, we can call modern optimization solvers, such as Adam, to iteratively try different bounds of ReLUs and refine the verification results. Then, as soon as it finds the bounds that suffice to certify the robustness of transformers, verification can be terminated.

We evaluate our approaches on transformer classifiers and TinyBERT~\cite{jiao-etal-2020-tinybert} trained with datasets about sentiment analysis. In total, we experiment on 10 transformer-based models with different complexities. 
We evaluate the robustness of each model with 140 verification tasks, resulting in 1400 verification tasks totally in our evaluation. We select the state-of-the-art approach~\cite{Shi2020Robustness} as our baseline and compare our approach with it in terms of verification precision and efficiency. By our  results, we observe that, for most of the verification tasks, our approaches, especially \toola, demonstrate evident precision improvement, which significantly enhances the reliability of verification results. 

\smallskip\noindent\textit{Impact}
Notably, this work not only delivers two effective verification approaches, but more importantly, by the use of ReLU, it bridges  verification methodologies between transformers and classic neural networks, potentially allowing more established analysis methods for ReLUs to be applied in transformer verification. In this regard, it paves the path to future research in transformer verification.

\section{Preliminaries}
\label{sec:background}

\subsection{Notations}\label{sec:notations}
In this paper, we adopt the notations as follows. Given $X\in \R^{n\times m}$ as a matrix, we use $X_{ij}\in \R$ (where $i\in \{1,\ldots, n\}$, $j\in \{1,\ldots,m\}$) to denote the element of $X$ with index $(i, j)$. We use $X_{i,:}\in \R^m$ to denote the $i$-th row of $X$, which is an $m$-dimensional vector. Moreover, if $Y\in \R^m$ is an $m$-dimensional vector, we use $Y_i\in\R$ to denote the element of $Y$ at its $i$-th dimension.

\subsection{Transformer Verification Problem}\label{sec:problemStatement}
\begin{wrapfigure}[13]{r}{0.5\textwidth}
\centering
\vspace{-2.5em}
\includegraphics[width=0.75\linewidth, alt={Transformer architecture}]{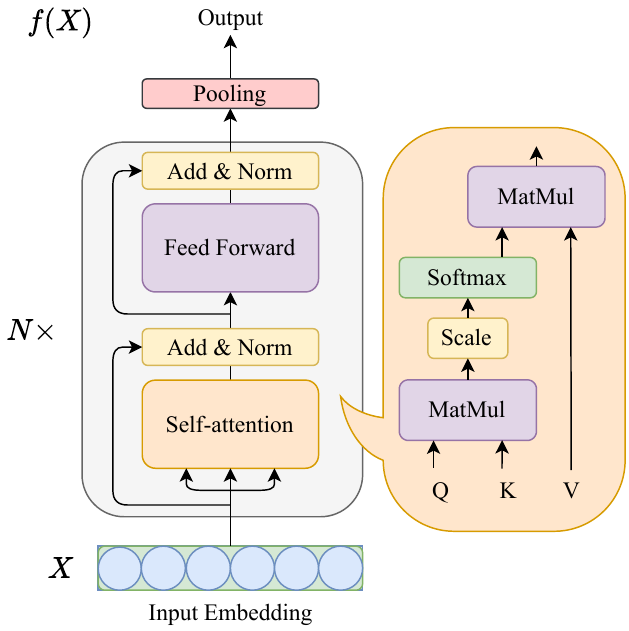}
\caption{Transformer architecture} \label{fig:architecture}
\end{wrapfigure}
%
\textit{Self-Attentive Transformers}
Transformers~\cite{NIPS2017_3f5ee243} feature the adoption of \emph{self-attention} mechanism, and have been adopted in various domains such as computer vision and natural language processing. 
While transformers typically consist of both encoders and decoders, in this paper we follow~\cite{Shi2020Robustness} and consider encoders only, which is adopted in architectures such as BERT~\cite{devlin-etal-2019-bert} and can be used for classification tasks in applications such as sentiment analysis. Due to limited space, we leave more detailed introduction to transformer architectures in \S{}\ref{sec:detailedTransformer} and here focus on the formal representations of transformer inferences.  

As visualized in Fig.~\ref{fig:architecture}, an $\modelLayer$-layer transformer is a function $\transformer\colon \R^{\numInputFirst\times \numInputSecond}\mapsto \R^{\numClass}$ that maps an input $\highX\in \R^{\numInputFirst\times \numInputSecond}$ to an output $\transformer(\highX)\in \R^{\numClass}$, as follows:
\begin{equation}\label{eq:transformer}
\begin{aligned}
\transformerSubOne(\highX) &\Defeq \faddnormone(\fattention(\highX\QWmatrix,\highX\KWmatrix,\highX\VWmatrix), \highX), \\
\transformerSubTwo(\highY) &\Defeq \faddnormtwo(\ffnn(\highY), \highY), \\
\transformer(\highX) &\Defeq \fpool(\underbrace{\transformerSubTwo\circ\transformerSubOne(\cdots(\transformerSubTwo\circ\transformerSubOne(\highX)))}_{\text{calling  $\transformerSubTwo\circ\transformerSubOne(\highX)$ for \modelLayer times}}) 
\end{aligned}
\end{equation}
The transformer first goes through \transformerSubOne, which includes the function $\fattention\colon \R^{\numInputFirst\times\numInputSecond}\times \R^{\numInputFirst\times\numInputSecond}\times \R^{\numInputFirst\times\numInputSecond} \mapsto \R^{\numInputFirst\times\numInputSecond}$  that represents the computations in self-attention layers. It takes as input three matrices \Qmat, \Kmat and \Vmat, each computed by multiplying the input \highX with a matrix (i.e., \QWmatrix, \KWmatrix, or \VWmatrix) of tuned parameters, and computes the output matrix $\fattention(\Qmat, \Kmat, \Vmat)\in \R^{\numInputFirst\times \numInputSecond}$ as follows: 
\begin{align}\label{eq:attention}
\fattention(\Qmat,\Kmat,\Vmat)\Defeq\softmax\left(\frac{\Qmat\transpose{\Kmat}}{\sqrt{d_{k}}}\right)\Vmat
\end{align}
where \softmax is a normalization function and $\sqrt{d_k}$ is a pre-defined rescaling factor. The output $\fattention(\Qmat, \Kmat, \Vmat)$ then goes through $\faddnormone\colon \R^{\numInputFirst\times\numInputSecond}\times \R^{\numInputFirst\times\numInputSecond} \mapsto \R^{\numInputFirst\times\numInputSecond}$ for a residual addition and normalization operation, deriving the result of $\transformerSubOne(\highX)$. 

Then, $\transformerSubOne(\highX)$ is taken as the input of \transformerSubTwo, and goes through  \ffnn that involves two fully-connected  layers and acts as a feed-forward neural network, each layer consisting of an affine transformation and a \relu activation function. After this, the output of \ffnn goes through another add\&norm layer \faddnormtwo to derive the output of \transformerSubTwo. In the whole transformer that contains $\modelLayer$ encoder layers, $\transformerSubTwo\circ\transformerSubOne(\highX)$ is called iteratively for $\modelLayer$ times, and then its output goes through a pooling layer \fpool that aggregates the information in the resulting matrix to a vector $\transformer(\highX)\in \R^\numClass$ as the final output of the transformer. 

\smallskip\noindent\textit{Robustness Verification Problem} Robustness refers to the behavioral consistency of transformers under adversarial perturbations and has been widely recognized as an essential property for their reliability. Consider a natural language processing task where $\highX$ represents a given sentence and each row $\highXEle{i,:}$ of \highX represents a word. We define a predicate  $\inSpec\colon \R^{\numInputFirst\times\numInputSecond}\mapsto\{\true, \false\}$ to identify the valid space of the perturbed input \highXPerturb, as follows:
\begin{math}
   \inSpec(\highXPerturb) \Defeq  \forall \perturbIdx\in \perturbIdxSet.\left(\norm[p]{\highXPerturbEle{d,:} - \highXEle{d,:}} \le \epsilon \right)
\end{math},
where $\perturbIdxSet\subseteq \{1,\ldots, \numInputFirst\}$ is a subset of the words selected to add perturbations. The predicate \inSpec imposes the following requirement to the perturbed input \highXPerturb: for each perturbed word $\highXPerturbEle{d,:}$, the distance from the original word $\highXEle{d,:}$ should be less than a threshold $\epsilon$ under $L_p$-norm.  

The robustness property requires that, for any perturbed input $\highXPerturb$ that holds $\inSpec(\highXPerturb)$, the output $\transformer(\highXPerturb)$ of the transformer should satisfy $\outSpec(\transformer(\highXPerturb))$, where $\outSpec\colon \R^\numClass\mapsto \{\true, \false\}$ is a predicate defined as follows:
\begin{align}\label{eq:outSpec}
    \textstyle\outSpec(\highY) \Defeq \left(\min_{1 \le i\le \numClass, i \neq i_l}\left(\highYEle{i_l} - \highYEle{i}\right) \right) > 0 
\end{align}
where $\highYEle{i}$ denotes the $i$-th element of the vector \highY, and $i_l$ is the label of $\highX$ inferred by the transformer, i.e., 
\begin{math}
    \forall i\in\{1,\ldots, \numClass\}. \transformer(\highX)_{i_l} - \transformer(\highX)_{i} > 0
\end{math}.

The robustness verification problem asks whether a transformer satisfies the robustness property, under a given input \highX and a given $\epsilon$. Moreover, as a common extended application of verification techniques, in this paper, we also aim to find the maximal $\epsilon$ under which the transformer remains robust.  Notably, this is meaningful in practice, e.g., the certified robustness within $\epsilon$ can relief the concern about any synonym substitution attack (see~\S{}\ref{sec:introduction})  within this region.

\subsection{Transformer Verification via Linear Relaxation}\label{sec:iclrApproach}
While verification for neural networks with simple architectures, such as feed-forward neural networks, has been extensively studied, verification for transformers still remains at a very early stage. The challenge arises from the high non-linearity of their inference function \transformer (see~\S{}\ref{sec:problemStatement}), which consists of multiple non-linear components such as self-attention layers. While pursuing an exact bound of \transformer is difficult, linear relaxation~\cite{Shi2020Robustness} provides a promising way to over-approximate the output range of \transformer, thereby achieving high efficiency. 

At the core of this approach are the construction and propagation of linear constraints used to bound the output ranges of non-linear operations throughout transformer inferences, which finally lead to linear bounds of transformer outputs. These linear bounds constitute over-approximation of transformer outputs, such that, as long as the over-approximation satisfies the robustness property, the original output must also satisfy it, thereby certifying robustness satisfaction.

In~\cite{Shi2020Robustness}, the linear constraints are constructed depending on the operations in each layer. As shown in Eq.~\ref{eq:transformer}, there are various non-linear operations in transformers, including self-attention in Eq.~\ref{eq:attention}, \relu activation functions, pooling layers, etc; as linear relaxation for some of these functions (e.g., \relu activation functions, pooling) have been extensively studied in classic neural network verification~\cite{singh2019abstract,zhang2018efficient,brix2023first}, we mainly focus on the self-attention layers uniquely existing in transformers and leave the relaxation of other functions in~\S{}\ref{sec:detailedICLR}. As shown in Eq.~\ref{eq:attention}, a self-attention layer mainly consists of two matrix multiplications and a softmax function. Below, we elaborate on the linear relaxation for the first matrix multiplication $\Qmat\transpose{\Kmat}$, for demonstration of our technique.

In Eq.~\ref{eq:attention}, given $\Qmat\in \R^{\numInputFirst\times \numInputSecond}$ and $\Kmat\in \R^{\numInputFirst\times \numInputSecond}$, $\Qmat\transpose{\Kmat}$ results in an $\numInputFirst\times \numInputFirst$ matrix, and each element of $\Qmat\transpose{\Kmat}$ can be represented as
\begin{math}
\Qmat\transpose{\Kmat}_{ij} =  \sum_{h=1}^{\numInputSecond}\qEle\kEle
\end{math}, 
given $i, j\in \{1,\ldots, \numInputFirst\}$. To compute the upper/lower bound of $\Qmat\transpose{\Kmat}_{ij}$, it suffices to compute the upper/lower bounds of a dot product $\qEle\kEle$, for each $i, j\in \{1,\ldots, \numInputFirst\}$ and $h\in\{1,\ldots, \numInputSecond\}$, and aggregate these bounds by taking their summation. To this end, \cite{Shi2020Robustness} proposes to use planar upper bound \qkUpperOne{\qEle, \kEle} and lower bound  \qkLowerOne{\qEle, \kEle} to bound the output range of $\qEle\kEle$, as follows:
\begin{equation}\label{eq:iclr}
\begin{aligned}
    \qkUpperOne{\qEle, \kEle} &\Defeq \kUpper\qEle + \qLower\kEle - \qLower\kUpper\\
    \qkLowerOne{\qEle, \kEle} &\Defeq \kLower\qEle + \qLower\kEle - \qLower\kLower 
\end{aligned}
\end{equation}
where $\qUpper$ and $\qLower$ are the upper and lower bounds of $\qEle$, and $\kUpper$ and $\kLower$ are the upper and lower bounds of $\kEle$. These bounds can be taken from the  bounds of the output of the preceding operations. For example, if the preceding operation of a self-attention layer is the input layer, $\Qmat$ simply involves an affine transformation to the input $\highXPerturb$ of the transformer (i.e., multiplying $\highXPerturb$ by $\QWmatrix$), hence the bounds for each element $\qEle$ can be computed by taking the affine transformation to the bounds for each element of $\highXPerturb$; similarly, we can compute the bounds for $\Kmat$. The soundness proof of Eq.~\ref{eq:iclr} can be found in~\cite[\S{}C.1]{Shi2020Robustness}.

By Eq.~\ref{eq:iclr}, we can compute the upper and lower bounds for each element of $\Qmat\transpose{\Kmat}$. In subsequent steps, these bounds are propagated layer by layer (see~\cite{Shi2020Robustness} and \S{}\ref{sec:detailedICLR} for details), until reaching the final transformer output $\transformer(\highXPerturb)$.  Based on the bounds for $\transformer(\highXPerturb)$, we can assess whether it holds the property $\outSpec(\transformer(\highXPerturb))$, which is returned as the verification result.

\section{ReLU-Based Bound Representation}
In linear relaxation approaches, imprecise bounding can introduce significant approximation error at each layer, which ultimately leads to \emph{false alarms}, i.e., the situations where the over-approximated output bounds violate the robustness property, but the actual output does not. To mitigate this issue, it is essential to construct tighter linear bounds that can reduce approximation error. Although the linear bounds in Eq.~\ref{eq:iclr} are provably sound, they are often suboptimal, thus significantly hindering the precision of verification results. In this section, we present a \relu-inspired abstraction refinement approach, which relies on a dual bound of Eq.~\ref{eq:iclr} for dot products in matrix multiplication.

\begin{figure}[!tb]
\centering
\begin{subfigure}[b]{0.35\linewidth}
\includegraphics[width=\textwidth, alt = {$z=xy$}]{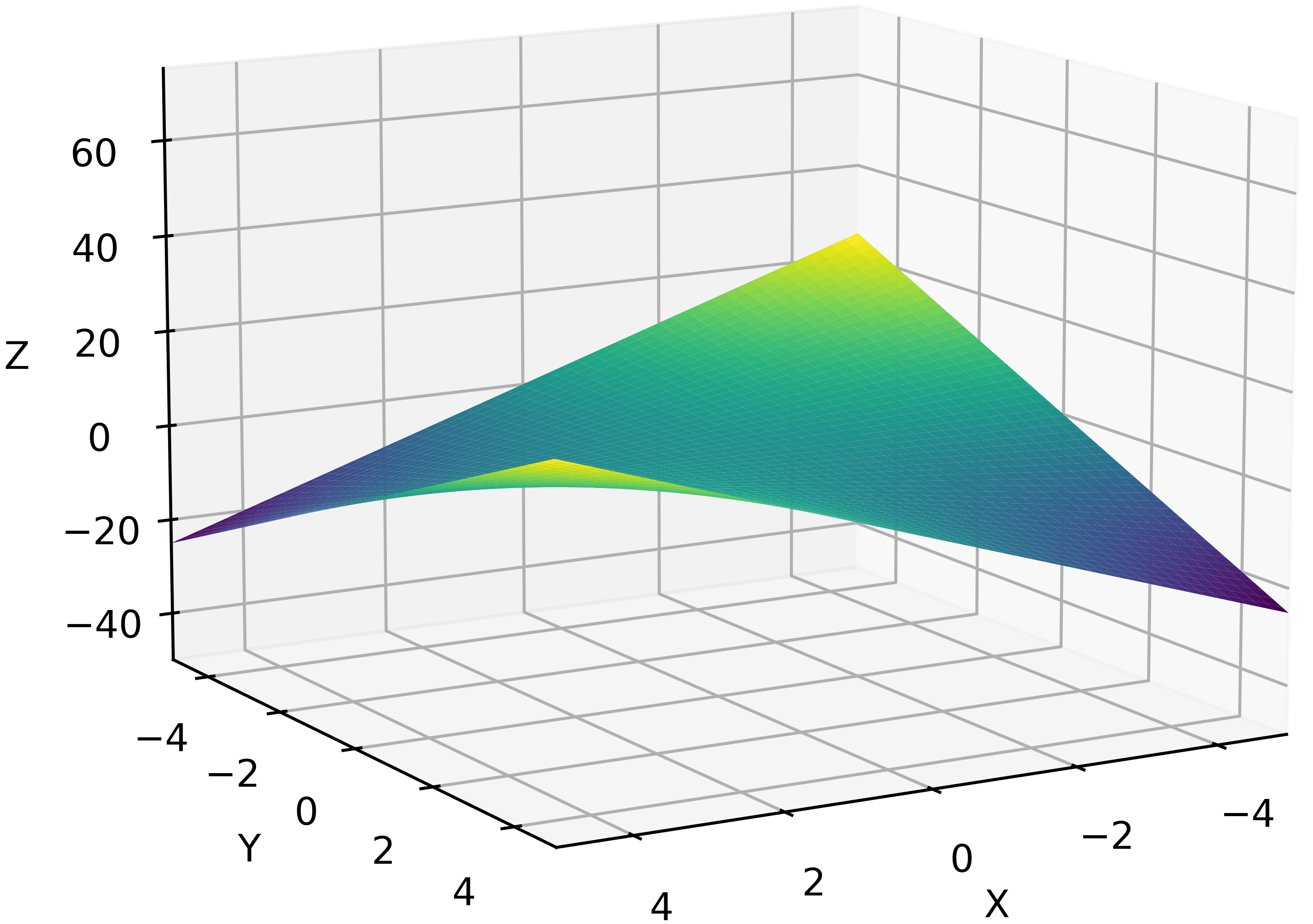}
\caption{$z=xy$}\label{fig:surface1}
\end{subfigure}
\qquad
\begin{subfigure}[b]{0.35\linewidth}
\includegraphics[width=\textwidth, alt = {Eq.~\ref{eq:iclr}: $\qkUpperOne{x, y}$}]{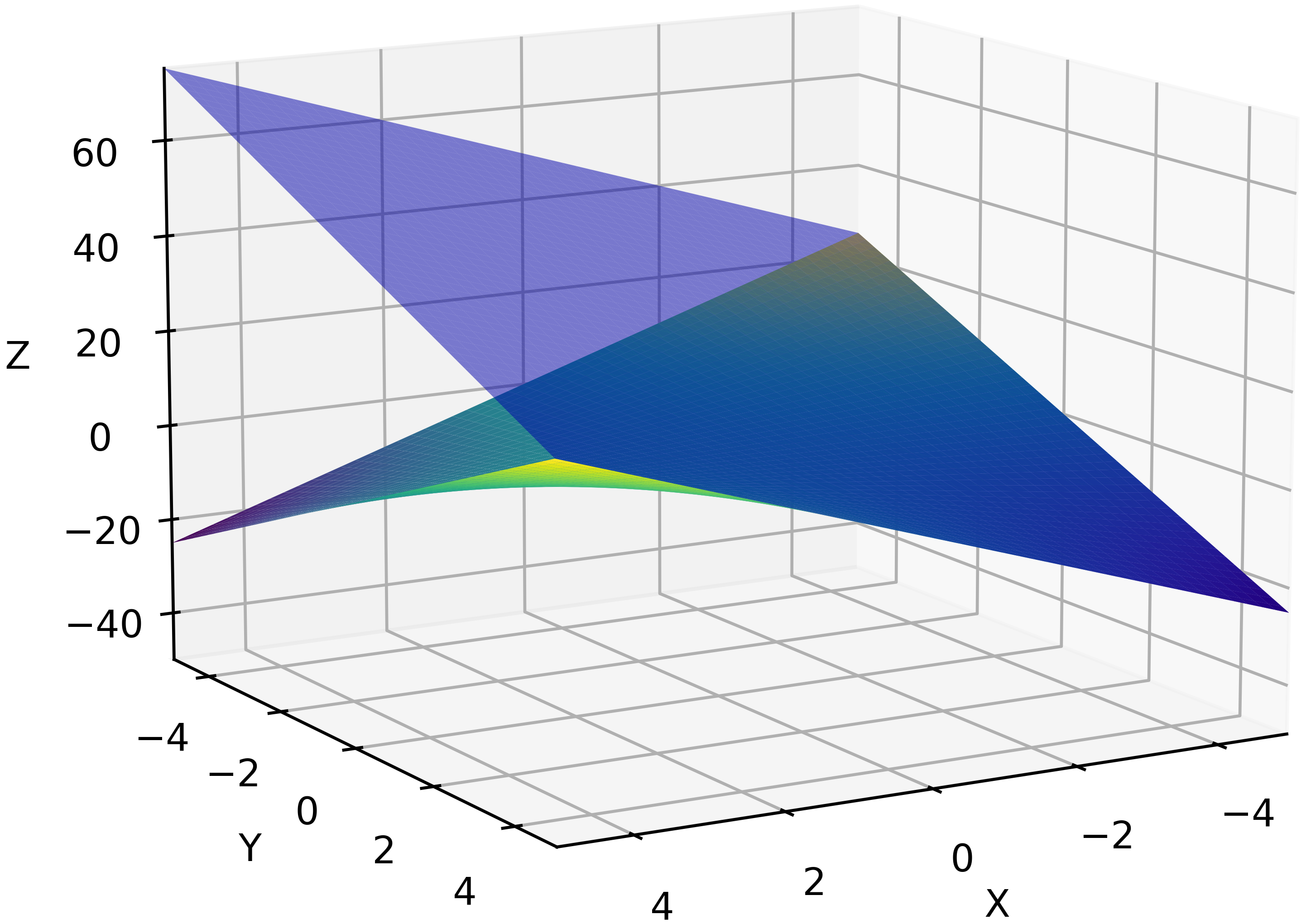}
\caption{Eq.~\ref{eq:iclr}: $\qkUpperOne{x, y}$}
\label{fig:surface2}
\end{subfigure}

\begin{subfigure}[b]{0.35\linewidth}
\includegraphics[width=\textwidth, alt = {Eq.~\ref{eq:dualBounds}: $\qkUpperTwo{x, y}$}]{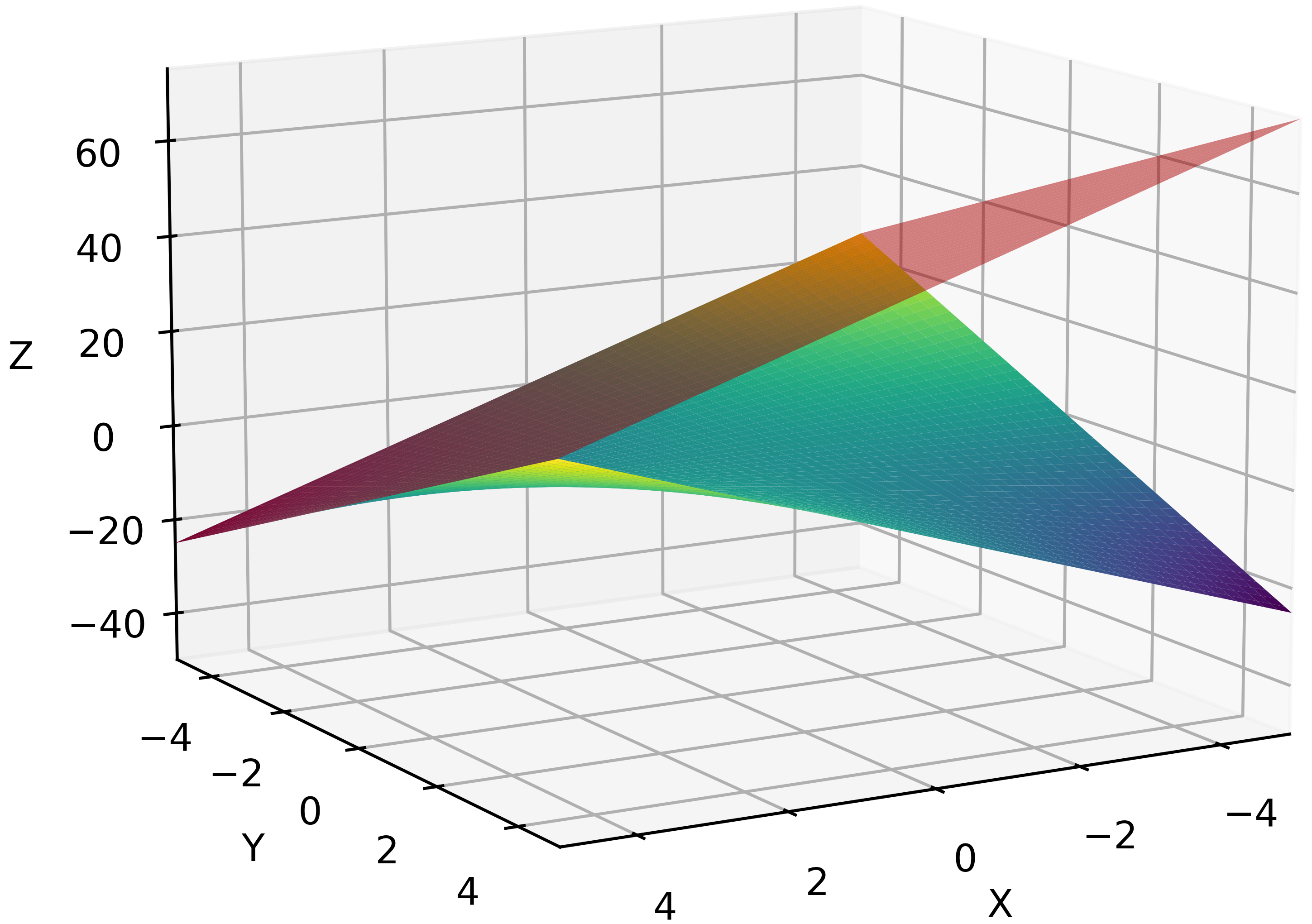}
\caption{Eq.~\ref{eq:dualBounds}: $\qkUpperTwo{x, y}$}
\label{fig:surface3}
\end{subfigure}
\qquad
\begin{subfigure}[b]{0.35\linewidth}
\includegraphics[width=\textwidth, alt = {Eq.~\ref{eq:jointU}: the merged bound}]{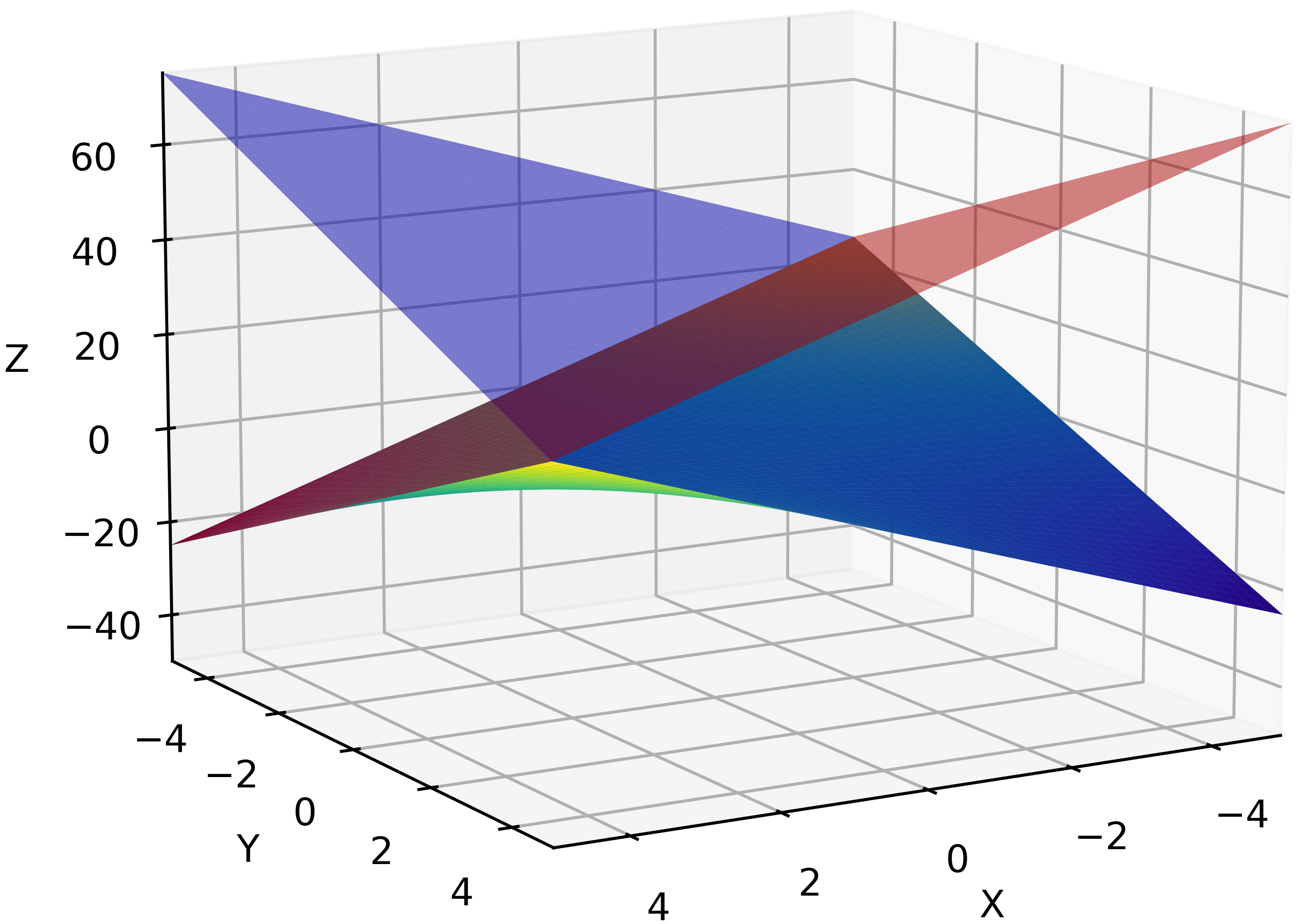}
\caption{Eq.~\ref{eq:jointU}: the merged bound}
\label{fig:surfaceAggregated}
\end{subfigure}
\caption{Illustrations of $z = xy$ and its different upper bounds}\label{fig:dotproductRelax}
\end{figure}

\subsection{Dual Bound Construction} Intuitively, as shown in Fig.~\ref{fig:surface1}, a dot product can form a hyperbolic paraboloid, and the bound given by Eq.~\ref{eq:iclr} essentially involves a plane that tightly bounds one side of the hyperbolic paraboloid. Dually, we can derive another plane that bounds the hyperbolic paraboloid from another side, i.e., we define another upper bound \qkUpperTwo{\qEle, \kEle} and lower bound \qkLowerTwo{\qEle, \kEle} for $\qEle\kEle$, as follows:
\begin{equation}\label{eq:dualBounds}
\begin{aligned}
    \qkUpperTwo{\qEle, \kEle} \Defeq \kLower \qEle + \qUpper \kEle - \qUpper\kLower\\
    \qkLowerTwo{\qEle, \kEle} \Defeq \kUpper\qEle + \qUpper\kEle - \qUpper\kUpper
\end{aligned}
\end{equation}
where $\qUpper$, $\qLower$, $\kUpper$,  $\kLower$ are the upper and lower bounds of $\Qmat$ and $\Kmat$, as defined in Eq.~\ref{eq:iclr}. The soundness of $\qkUpperTwo{\qEle, \kEle}$ and $\qkLowerTwo{\qEle, \kEle}$ is stated in Lemma~\ref{lem:dualBounds}, and we leave the proof in~\S{}\ref{sec:lemmaProof}.
\begin{lemma}\label{lem:dualBounds}
   Given the definition in Eq.~\ref{eq:dualBounds}, for any $\qEle\in \R$ and $\kEle\in\R$, it holds that  $\qEle\kEle \le \qkUpperTwo{\qEle, \kEle} $, and $\qEle\kEle \ge \qkLowerTwo{\qEle, \kEle}$.
\end{lemma}

As illustrated in Fig.~\ref{fig:surface3}, the plane derived from Eq.~\ref{eq:dualBounds} complements the suboptimal approximation of Eq.~\ref{eq:iclr}, and consequently, the joint use of both Eq.~\ref{eq:iclr} and Eq.~\ref{eq:dualBounds}, based on the comparison of their values, results in tighter upper and lower bounds \qkUpperNonlin{\qEle, \kEle} and \qkLowerNonlin{\qEle, \kEle}, defined as follows:
\begin{align}
    &\qkUpperNonlin{\qEle, \kEle} \Defeq \min\left(\qkUpperOne{\qEle, \kEle}, \qkUpperTwo{\qEle, \kEle}\right) \label{eq:jointU} \\
    &\qkLowerNonlin{\qEle, \kEle} \Defeq \max\left(\qkLowerOne{\qEle,\kEle}, \qkLowerTwo{\qEle,\kEle}\right) \label{eq:jointL}
\end{align}

While Eq.~\ref{eq:jointU} and Eq.~\ref{eq:jointL} significantly refine the approximation error compared to using a single plane such as the one in Eq.~\ref{eq:iclr} or in Eq.~\ref{eq:dualBounds} only, adopting them for verification is not feasible, because they are no longer planar but non-linear (see Fig.~\ref{fig:surfaceAggregated}); in that case, propagating the bounds throughout the transformer operations requires to deal with exponentially increasing planar bounds, making verification time also exponentially increasing.  Consequently, this approach would not be feasible to handle transformers of large sizes. To address this issue,  we introduce our approach by identifying new linear bounds based on Eq.~\ref{eq:jointU} and Eq.~\ref{eq:jointL}, which can achieve higher precision without compromise of efficiency. 

\subsection{ReLU-Based Representation of Dual Bounds}\label{sec:reluRepre}
In this section, we present our ReLU-based representations of Eq.~\ref{eq:jointU} and Eq.~\ref{eq:jointL}, by which we derive our proposed linear bounds for dot products. The advantage of our representations involves that, we can exploit the rich body of existing research in linear relaxation of ReLU activation functions for classic neural network verification, thereby refining the approximation error in the existing approach~\cite{Shi2020Robustness} to achieve better verification precision. 

As a piecewise linear function, ReLU has been extensively adopted and studied as activation functions in various types of neural networks, e.g., feed-forward neural networks, the most fundamental neural network architecture. In our context, we explore a \emph{novel usage} of \relu, namely, we represent \qkUpperNonlin{\qEle,\kEle} in Eq.~\ref{eq:jointU} and \qkLowerNonlin{\qEle, \kEle} in Eq.~\ref{eq:jointL} by using ReLU, formulated as follows: 
\begin{align}
    &\qkUpperNonlin{\qEle,\kEle} \Defeq \qkUpperOne{\qEle,\kEle} - \relu\left(\qkUpperOne{\qEle, \kEle} - \qkUpperTwo{\qEle, \kEle}\right)\label{eq:reluU} \\
    &\qkLowerNonlin{\qEle, \kEle} \Defeq \qkLowerOne{\qEle,\kEle} + \relu\left(\qkLowerTwo{\qEle,\kEle} - 
 \qkLowerOne{\qEle, \kEle}\right) \label{eq:reluL}
\end{align}

The equivalence between Eq.~\ref{eq:jointU} and Eq.~\ref{eq:reluU}, and dually, between Eq.~\ref{eq:jointL} and Eq.~\ref{eq:reluL}, can be derived based on the definition of \relu, which we leave in~\S{}\ref{sec:equal}. 

\medskip
Next, we explain how Eq.~\ref{eq:jointU} and Eq.~\ref{eq:jointL} are useful to explain the nature of the approximation in~\cite{Shi2020Robustness}, thereby proposing our new approach that refines the existing approach by leveraging the ReLUs in Eq.~\ref{eq:jointU} and Eq.~\ref{eq:jointL}.

\smallskip\noindent\textit{Lower Bounds of ReLU} At this stage, our aim is to obtain a linear upper bound \qkUpperLin{\qEle,\kEle} for \qkUpperNonlin{\qEle, \kEle} in Eq.~\ref{eq:reluU} and a linear lower bound \qkLowerLin{\qEle,\kEle} for \qkLowerNonlin{\qEle, \kEle} in Eq.~\ref{eq:reluL}, because
\begin{inparaenum}[1)]
    \item both $\qkUpperLin{\qEle,\kEle}$ and $\qkLowerLin{\qEle,\kEle}$ should be sound;
    \item both Eq.~\ref{eq:reluU} and Eq.~\ref{eq:reluL} are non-linear.
\end{inparaenum} 

For both \qkUpperNonlin{\qEle, \kEle} and \qkLowerNonlin{\qEle, \kEle}, the non-linearity arises from ReLUs, so the problem boils down to finding linear bounds for ReLUs to substitute their occurrences in Eq.~\ref{eq:reluU} and Eq.~\ref{eq:reluL}. Moreover, we find that, in both Eq.~\ref{eq:reluU} and Eq.~\ref{eq:reluL}, only the linear lower bound $\frelulow$ of ReLU is needed, because it suffices to ensure the soundness of \qkUpperLin{\qEle,\kEle} and \qkLowerLin{\qEle,\kEle}, shown as follows:
\begin{equation}
\begin{aligned}
    \qkUpperNonlin{\qEle,\kEle} &\le  \qkUpperOne{\qEle,\kEle} - \frelulow\left(\qkUpperOne{\qEle, \kEle} - \qkUpperTwo{\qEle, \kEle}\right) \\
    &= \qkUpperLin{\qEle,\kEle} \label{eq:linU} 
\end{aligned}
\end{equation}
\begin{equation}
\begin{aligned}
    \qkLowerNonlin{\qEle,\kEle} &\ge   \qkLowerOne{\qEle,\kEle} + \frelulow\left(\qkLowerTwo{\qEle,\kEle} - 
 \qkLowerOne{\qEle, \kEle}\right) \\
 &= \qkLowerLin{\qEle,\kEle} \label{eq:linL}
\end{aligned}
\end{equation}
Consequently, we need to find an optimal linear bound $\frelulow$ for ReLU that can minimize the approximation error to improve the precision of verification. 

Finding linear bounds for ReLUs is an important issue and has been extensively studied~\cite{singh2019abstract,zhang2018efficient,xu2020fast,weng2018towards,zhang2022provably} in classic neural network verification (i.e., verification of feed-forward neural networks with fully-connected layers). 
In general, given the definition of ReLU, the lower bound $\frelulow(x)$ of ReLU can be identified by a parametric function:
\begin{align}
    \frelulow(x) = \alpha\cdot x, \text{ where } \alpha\in[0, 1] \label{eq:reluLB}
\end{align}
The selection of the parameter $\alpha$ is a tricky question and has a significant impact to the approximation precision. The approach for deciding it constitutes our main technical contribution, which will be elaborated on in \S{}\ref{sec:methodology}.

\smallskip\noindent\textit{Explanation of Existing Approach} We show how the ReLU-based bound representation in Eq.~\ref{eq:reluU} and Eq.~\ref{eq:reluL} can be used to explain the approximation approach in~\cite{Shi2020Robustness}, highlighting the connection between transformer verification and classic neural network verification. Specifically, when $\alpha = 0$, by Eq.~\ref{eq:linU}, Eq.~\ref{eq:linL} and Eq.~\ref{eq:reluLB}, we can derive the following:
\begin{align*}
    \qkUpperLin{\qEle, \kEle} = \qkUpperOne{\qEle, \kEle}, \quad \qkLowerLin{\qEle, \kEle} = \qkLowerOne{\qEle, \kEle}
\end{align*} 
Namely, $\alpha = 0$ is essentially the selection in~\cite{Shi2020Robustness}. Dually, when $\alpha = 1$, by Eq.~\ref{eq:linU}, Eq.~\ref{eq:linL} and Eq.~\ref{eq:reluLB}, we can derive that: 
\begin{align*}
    \qkUpperLin{\qEle, \kEle} = \qkUpperTwo{\qEle, \kEle}, \quad \qkLowerLin{\qEle, \kEle} = \qkLowerTwo{\qEle, \kEle}
\end{align*}
Namely, $\alpha = 1$ is essentially the selection of Eq.~\ref{eq:dualBounds}. 

However, existing works~\cite{weng2018towards,zhang2018efficient,zhang2022provably} have shown that, none of the above strategies, i.e., by fixing $\alpha$ to be $1$ or to be $0$, is the optimal way of ReLU relaxation. This motivates us to explore more effective approaches to refine the approximation.

\section{Our Proposed Abstraction Refinement Approach}
\label{sec:methodology}
As stated in \S{}\ref{sec:reluRepre}, given the ReLU-based representation in Eq.~\ref{eq:reluU} and Eq.~\ref{eq:reluL}, we can reduce the problem of finding optimal planar bounds for dot products, to the problem of finding optimal $\alpha$, i.e., the slope of linear lower bounds of ReLUs. To this end, we extend two approaches from classic neural network verification: the first strategy inspired by~\cite{singh2019abstract} selects $\alpha$ based on the input ranges of ReLUs, with the aim of minimizing the approximation error, measured by the area of the over-approximation; the second one inspired by~\cite{xu2020fast} formalizes the approximation refinement as an optimization problem with $\alpha$ as decision variables, and calls modern optimization solvers to iteratively refine the approximation until verifying the problem. In the following sections, we elaborate on these two approaches. 
\begin{figure}[!tb]
\centering
\begin{subfigure}[b]{0.27\textwidth}
\includegraphics[width=\textwidth, alt = {ReLU function}]{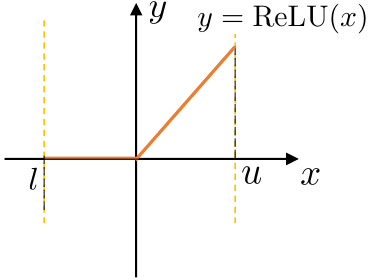}
\caption{ReLU function}\label{fig:reluOrig}
\label{fig:landscape}
\end{subfigure}~~~~~~%
\qquad
\begin{subfigure}[b]{0.22\textwidth}
\includegraphics[width=\textwidth, alt = {linear bounds when $\alpha = 0$}]{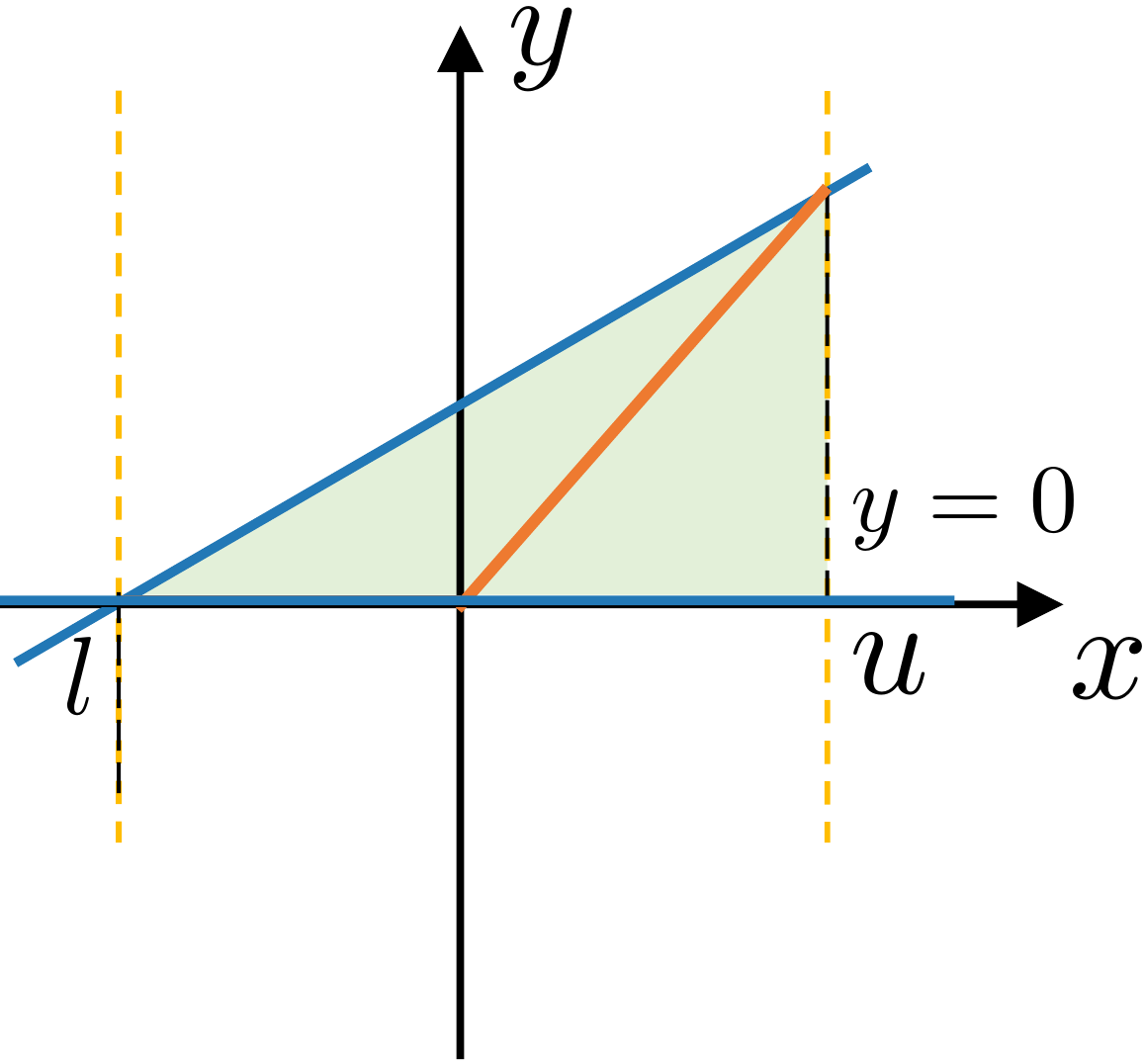}
\caption{$\alpha = 0$}
\label{fig:relu2}
\end{subfigure}
\quad
\begin{subfigure}[b]{0.22\textwidth}
\includegraphics[width=\textwidth, alt = {linear bounds when $\alpha = 1$}]{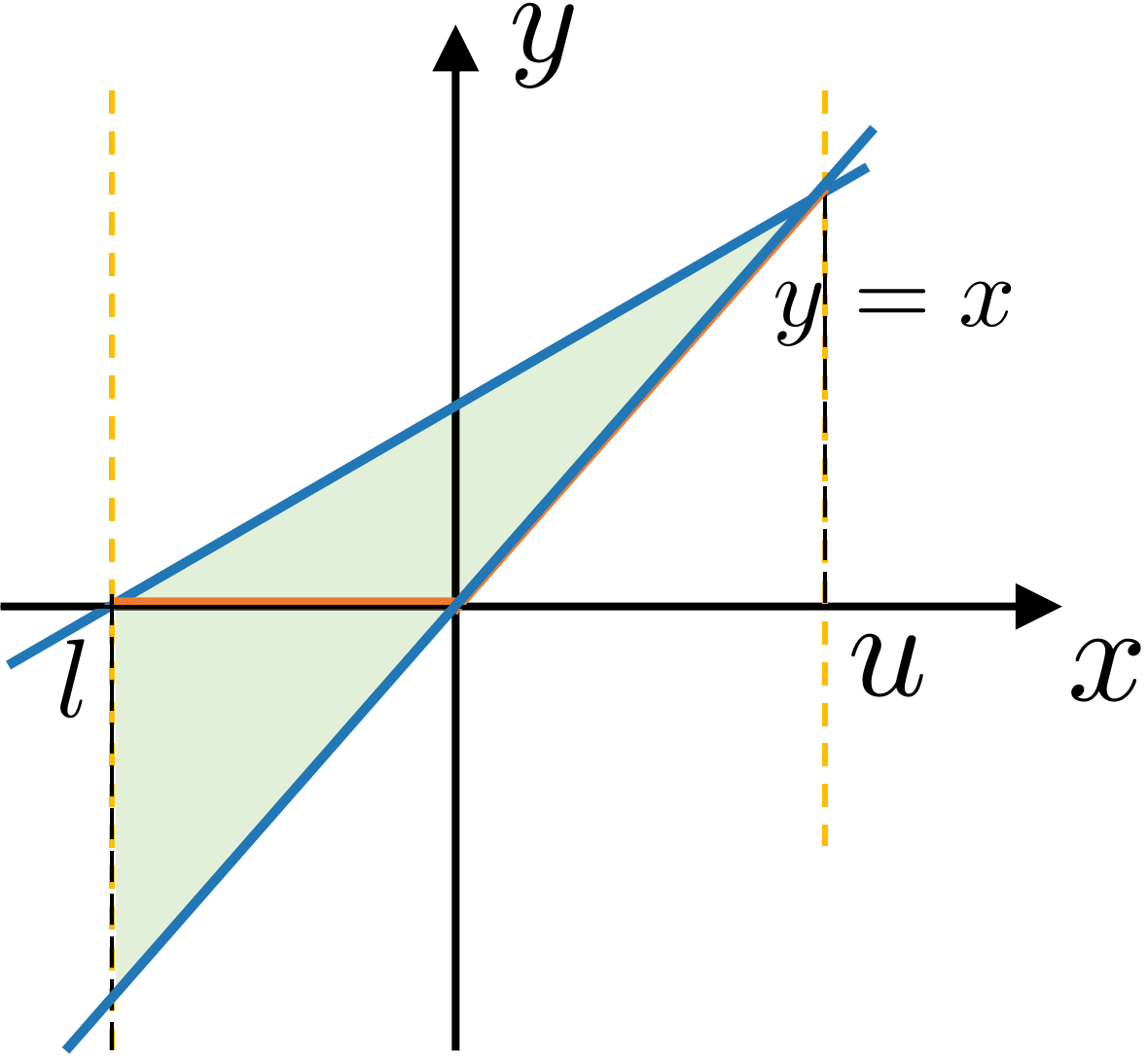}
\caption{$\alpha = 1$}
\label{fig:relu3}
\end{subfigure}
\caption{Linear relaxation of ReLU. Given the representation $y = \alpha x$ of the lower bound of ReLU, the sub-figures plot the different approximation errors (denoted as the shadow area) under different $\alpha$.}\label{fig:reluRelax}
\end{figure}
\subsection{\toolp: Rule-Based $\alpha$ Selection}\label{sec:ruleBased}
Selecting proper $\alpha$ in Eq.~\ref{eq:reluLB} as the lower bound of ReLU is an important but challenging problem in classic neural network verification. Depending on the input range of ReLU, there could be distinct approximation error due to the selection of $\alpha$. Here, the approximation error is often evaluated according to the area of the over-approximation brought by the linear bounds. 

Let $\reluInUpp$ and $\reluInLow$ be the upper and lower bounds of the input range of a ReLU.  In our context, for the upper bounds of dot products $\qEle\kEle$, by Eq.~\ref{eq:reluU} and Eq.~\ref{eq:iclr} \&\ref{eq:dualBounds}, $\reluInUpp$ and $\reluInLow$ are the upper and lower bounds of the following ReLU input:
\begin{small}
\begin{align}\label{eq:reluInputU}
  \qkUpperOne{\qEle, \kEle} - \qkUpperTwo{\qEle, \kEle} =  
  (\kUpper - \kLower)\qEle + (\qLower - \qUpper)\kEle - \qLower\kUpper + \qUpper\kLower
\end{align}
\end{small}
Similarly, we can also derive for the lower bounds of dot products by Eq.~\ref{eq:reluL} and Eq.~\ref{eq:iclr} \&\ref{eq:dualBounds}, in which case the ReLU input is as follows:
\begin{small}
\begin{align}\label{eq:reluInputL}
  \qkLowerTwo{\qEle, \kEle} - \qkLowerOne{\qEle, \kEle} = 
  (\kUpper - \kLower)\qEle + (\qUpper - \qLower)\kEle - \qUpper\kUpper + \qLower\kLower
\end{align}
\end{small}

Given Eq.~\ref{eq:reluInputU} and Eq.~\ref{eq:reluInputL}, we can derive the values of $\reluInUpp$ and $\reluInLow$, by computing the upper and lower bounds of these ReLU inputs. For instance, the upper bound $\reluInUpp$ and the lower bound $\reluInLow$ of Eq.~\ref{eq:reluInputU} equals to:
\begin{equation}\label{eq:reluInputBounds}
\begin{aligned}
    \reluInUpp = (\kUpper - \kLower)\qEleU + (\qLower - \qUpper)\kEleL - \qLower\kUpper + \qUpper\kLower \\
    \reluInLow = (\kUpper - \kLower)\qEleL + (\qLower - \qUpper)\kEleU - \qLower\kUpper + \qUpper\kLower
\end{aligned}
\end{equation}
where \qEleU and \qEleL are respectively the upper and lower bounds of  \qEle, and \kEleU and \kEleL are respectively the upper and lower bounds of \kEle. As explained in~\S{}\ref{sec:iclrApproach}, these bounds can be derived from the preceding operations in transformer inferences, hence we can obtain the value of $\reluInUpp$ and $\reluInLow$. The bounds for Eq.~\ref{eq:reluInputL} can be obtained similarly.

Below, we consider how to choose $\alpha$ in the case when $\reluInUpp\in(0, \infty)$ and $\reluInLow \in (-\infty, 0)$ only; in other cases (i.e., when $\reluInUpp\in(-\infty, 0)$ or $\reluInLow \in (0, \infty)$), the ReLU becomes linear and so we do not need to choose $\alpha$. 

As shown in Fig.~\ref{fig:relu2} and \ref{fig:relu3}, the area of over-approximation is $\frac{|\reluInUpp|^2 + |\reluInUpp||\reluInLow|}{2}$ if $\alpha$ is set to be 0, and the area of over-approximation is  $\frac{|\reluInLow|^2 + |\reluInUpp||\reluInLow|}{2}$ if $\alpha$ is set to be 1. Therefore, \cite{singh2019abstract} introduces the following condition-based policy:
\begin{compactenum}[i)]
    \item \label{case:alpha0} when $|\reluInLow| > |\reluInUpp|$, $\alpha$ is set to be $0$, because the area of over-approximation when $\alpha = 0$ is smaller than that when $\alpha = 1$;
    \item \label{case:alpha1} when $|\reluInUpp| > |\reluInLow|$, $\alpha$ is set to be $1$, because the area of over-approximation when $\alpha = 1$ is smaller than that when $\alpha = 0$.
\end{compactenum}
     Namely, this policy takes the better selection in both of the cases.

In our context, as we can compute $\reluInUpp$ and $\reluInLow$ for each ReLU, we can apply the same policy for selecting proper $\alpha$: 
\begin{inparaenum}[1)]
    \item when $|\reluInLow| > |\reluInUpp|$, we set $\alpha$ to be 0, which essentially equals to adopting the linear bounds in Eq.~\ref{eq:iclr};
    \item when $|\reluInUpp| > |\reluInLow|$, we set $\alpha$ to be 1, which essentially equals to adopting the linear bounds in Eq.~\ref{eq:dualBounds}.
\end{inparaenum}

\subsection{\toola: Optimization-Based Iterative Refinement}\label{sec:optBased}
While the rule-based strategy in \S{}\ref{sec:ruleBased} can refine the approximation, it only offers two choices for bounding each ReLU, i.e., $\alpha = 0$ or $\alpha = 1$, which may not be the optimal choices under all different cases. To further refine the approximation, we propose an optimization-based approach \toola, which treats $\alpha$ as the decision variables and iteratively searches for the optimal $\alpha$ to optimize the refinement. 

Recall the specification of transformer verification formalized in Eq.~\ref{eq:outSpec}; as we require that
\begin{math}
     \left(\min_{1 \le i\le \numClass, i \neq i_l}\left(\highYEle{i_l} - \highYEle{i}\right) \right) > 0 
\end{math}, 
verification consists in showing that the lower bound \marginLow of \margin, as defined below, is greater than 0. 
 \begin{align}\label{eq:margin}
      \margin \Defeq \left(\min_{1 \le i\le \numClass, i \neq i_l}\left(\highYEle{i_l} - \highYEle{i}\right) \right) 
 \end{align}

The bounds of \margin are decided by the bounds of \highY, i.e., the output vector of a given transformer. Further, the bounds of \highY can be obtained by layer-by-layer propagation, depending on the operations in the transformer. Below, we only elaborate on the computation of the bounds for dot products in self-attention layers, as we have done in previous sections; for other layers, we adopt the bounding methods in literature as introduced in ~\S{}\ref{sec:detailedICLR}. 

\medskip
Given $\frelulow(x) = \alpha\cdot x$ as a linear lower bound of ReLU in Eq.~\ref{eq:reluLB}, we can expand $\qkUpperLin{\qEle,\kEle}$ in Eq.~\ref{eq:linU} as follows:
\begin{equation}\label{eq:alphaUpper}
\begin{subfootnotesize}
\begin{aligned}
    \qkUpperLin{\qEle, \kEle} &= \qkUpperOne{\qEle,\kEle} - \frelulow\left(\qkUpperOne{\qEle, \kEle} - \qkUpperTwo{\qEle, \kEle}\right) \\
    &= (1-\alpha)\qkUpperOne{\qEle, \kEle} + \alpha\qkUpperTwo{\qEle, \kEle} \\
    &= \left((1- \alpha)\kUpper + \alpha\kLower\right)\qEle + \left((1 - \alpha)\qLower + \alpha \qUpper\right)\kEle - \left((1-\alpha)\qLower\kUpper + \alpha\qUpper\kLower\right)
\end{aligned}
\end{subfootnotesize}
\end{equation}
Similarly, we can expand $\qkLowerLin{\qEle,\kEle}$ as follows:
\begin{equation}\label{eq:alphaLower}
\begin{subfootnotesize}
\begin{aligned}
   \qkLowerLin{\qEle,\kEle} &= \qkLowerOne{\qEle,\kEle} + \frelulow\left(\qkLowerTwo{\qEle,\kEle} - 
 \qkLowerOne{\qEle, \kEle}\right) \\
 &= \alpha \qkLowerTwo{\qEle,\kEle}  + (1-\alpha) \qkLowerOne{\qEle,\kEle} \\
 &= \left( \alpha\kUpper + (1-\alpha)\kLower \right)\qEle + \left( \alpha\qUpper + (1-\alpha)\qLower \right)\kEle - \left( \alpha\qUpper\kUpper + (1-\alpha)\qLower\kLower \right)
\end{aligned}
\end{subfootnotesize}
\end{equation}
In this way, $\qkUpperLin{\qEle, \kEle}$ and $\qkLowerLin{\qEle,\kEle}$ are two parametric planar bounds that can be propagated throughout the layers of transformers, finally participating in the computation of transformer output $\highY$. 

\smallskip\noindent\textit{Optimization Approach} Now we  present our optimization-based approach. In this optimization, we treat the bound computations in Eq.~\ref{eq:alphaUpper} and Eq.~\ref{eq:alphaLower} for dot products, as well as the bound computations for other operations, as constraints. 
Then, as the usual practice in optimization (e.g., neural network training) and for the sake of numerical stability, we adopt the standard softplus/logistic loss function as the objective of optimization: 
\begin{align}
\loss \Defeq \log(1+\exp(-\,\marginLow)). \label{eq:logistic-loss}
\end{align}
This leads to an optimization problem as follows:
\begin{align*}
    \min \loss, \quad \text{s.t. }\forall\alpha\in [0,1]
\end{align*}
In this optimization, as the value of $\alpha$ for each ReLU varies, the loss function $\loss$ also varies, because the variation of $\alpha$ can change the upper and lower bounds of each dot product operation, which subsequently affects the upper and lower bounds of the input of the succeeding operation, i.e., the ReLU input of the succeeding operation (see Eq.~\ref{eq:reluInputBounds} for an explanation). The optimization can be terminated as soon as the value of \marginLow is positive, which signifies that such a group of $\alpha$ has been found that can ensure \margin to be positive, thereby certifying the robustness satisfaction of the transformer. 

One notable advantage of our formalization consists in that, the derivative of the loss w.r.t. each $\alpha$ can be easily computed by the automatic differentiation features in modern machine learning platforms (e.g., PyTorch), so we can call off-the-shelf optimization solvers, such as Adam, a widely used gradient-based optimization solver, to efficiently solve this problem. Consequently, this allows us to optimize $\alpha$ in a manner similar to neural network training, leveraging the extensive advances and optimizations developed in that domain.

\begin{theorem}[Soundness]
    Both \toolp and \toola are sound.
\end{theorem}

\noindent{\it Proof sketch} As long as $\alpha \in [0, 1]$, by Eq.~\ref{eq:reluU}-\ref{eq:linL}, the derived bounds for each dot product are sound over-approximation, hence both approaches are sound.

\section{Experimental Evaluation}
\label{sec:experiments}
We present the experimental evaluation of our proposed approach. We implement our approach based on the library of \crownBAF~\cite{Shi2020Robustness}.  All our code and data have been made publicly available in~\cite{liu_2026_artifact}. Extended evaluation results can be found in~\S{}\ref{sec:horizontal-lse}.

\subsection{Experimental Setup}
\label{sec:experimentsSetup}

\smallskip\noindent\textit{Datasets}
We evaluate our verification approaches with the models trained on  two datasets that have been widely-used for sentiment analysis, namely, the Stanford Sentiment Treebank (\sst)~\cite{zhang2015character} and the Yelp polarity dataset (\yelp)~\cite{socher2013recursive}. We leave more details about these two datasets, including the number of sentences used for training, validation, and testing, in~\S{}\ref{sec:extendedExperimentSettings}.

\begin{table}[!tb]
\centering
\caption{Details of adopted models} 
\label{table:modelDetail}
\setlength{\tabcolsep}{8pt}
      \begin{tabular}{lcc}
			\toprule
			Parameters      & Standard  & TinyBERT    \\
			\midrule
			\# Layers       & 1, 2, 3, 6 & 4   \\
			Attention heads & 8     & 12       \\
			Activation      & ReLU  & ReLU      \\
			Hidden size     & 512   & 312/1200      \\
          \multirow{2}{*}{ Accuracy (\%)}  & 82.3$\sim$83.2 (\sst)  & 80.9 (\sst) \\
          & 91.2$\sim$91.4 (\yelp)  & 91.5 (\yelp)   \\
			\bottomrule
		\end{tabular}
\end{table}
\smallskip\noindent\textit{Model Architectures}
Our experiments utilize transformer  architectures consisting of encoder layers only, following~\cite{Shi2020Robustness}, as detailed in Table~\ref{table:modelDetail}.
We adopt standard transformers  consisting of different numbers of encoder layers for both  \sst and \yelp; the number \modelLayer of model layers can be $\modelLayer\in \{1, 2, 3, 6\}$.
The models share the same number of attention heads (8 heads), activation functions (ReLU), and employ normalization without variance calculation.
The hidden sizes for both attention and fully-connected layers are fixed to be 512.
To exemplify the capability of our approaches in more practical contexts,  
we also evaluate our approach on TinyBERT~\cite{jiao-etal-2020-tinybert}, a compact encoder-only transformer that has been widely adopted in NLP for language model pre-training~\cite{bhargava-etal-2021-generalization,devlin-etal-2019-bert,NIPS2017_3f5ee243,xiao2023introductiontransformersnlpperspective,zhu2022robustness}.

\smallskip\noindent\textit{Robustness Properties}
In NLP, robustness is an important property for models' reliability; however, threats to robustness may arise from various types of attacks, such as synonym substitution~\cite{chiang-lee-2023-synonym} and tokenization variations~\cite{wegmann-etal-2025-tokenization}. Certifying robustness can eliminate the concern about these attacks within a given $\epsilon$-ball and ensures the reliability of model deployment; therefore in our evaluation, we assess the effectiveness of \tool in certifying models' robustness. 
Specifically, we adopt $L_1$ norm as the distance metric for inputs, while our approach is not limited to this norm but also works for others, in line with~\cite{Shi2020Robustness}. 
We adopt single-word perturbations following~\cite{Shi2020Robustness}, and
for each model,  we evaluate on 140 tasks, resulting in 1400 verification tasks in our evaluation in total.

\smallskip\noindent\textit{Research Questions} We evaluate our proposed approaches by answering the following three research questions:
\begin{compactitem}
    \item {\bf RQ1:} Can our approaches achieve higher  precision compared to the baseline?
    \item {\bf RQ2:} Is changing $\alpha$ useful to refine the approximation?
    \item {\bf RQ3:} How efficient are our approaches?
\end{compactitem}

\smallskip\noindent\textit{Baseline and Evaluation Metrics} To evaluate the performances of our approach, we select \crownBAF~\cite{Shi2020Robustness} as our baseline approach, and compare with it according to the following two metrics:
\begin{compactitem}
    \item {\it Maximal verified $\epsilon$}: As both the baseline and our approach are based on approximation, given an $\epsilon$, both approaches can throw false alarms. Nevertheless, as both approaches are sound, the larger an $\epsilon$ can be verified, the higher precision the approach can achieve. For the search of $\epsilon$, we follow literature~\cite{Shi2020Robustness} and perform a binary search with details left in~\S{}\ref{sec:binarySearch}.
    \item {\it Time costs}: For each approach, we record the time costs to reach the verification result, as an indicator for the efficiency of the approach. 
\end{compactitem}

\smallskip\noindent\textit{Hardware Settings}
All experiments are conducted using two NVIDIA A6000 GPUs operating in parallel.
Our implementations are built upon TensorFlow 1.14 and PyTorch 1.13, with CUDA version 10.2.

\subsection{Evaluation Results}
\researchquestion{Can our approaches achieve higher  precision compared to the baseline?}

\begin{figure}[!tb]
\centering
\begin{subfigure}[b]{0.3\textwidth}
\includegraphics[width=\textwidth, alt = {\sst, $\modelLayer = 1$}]{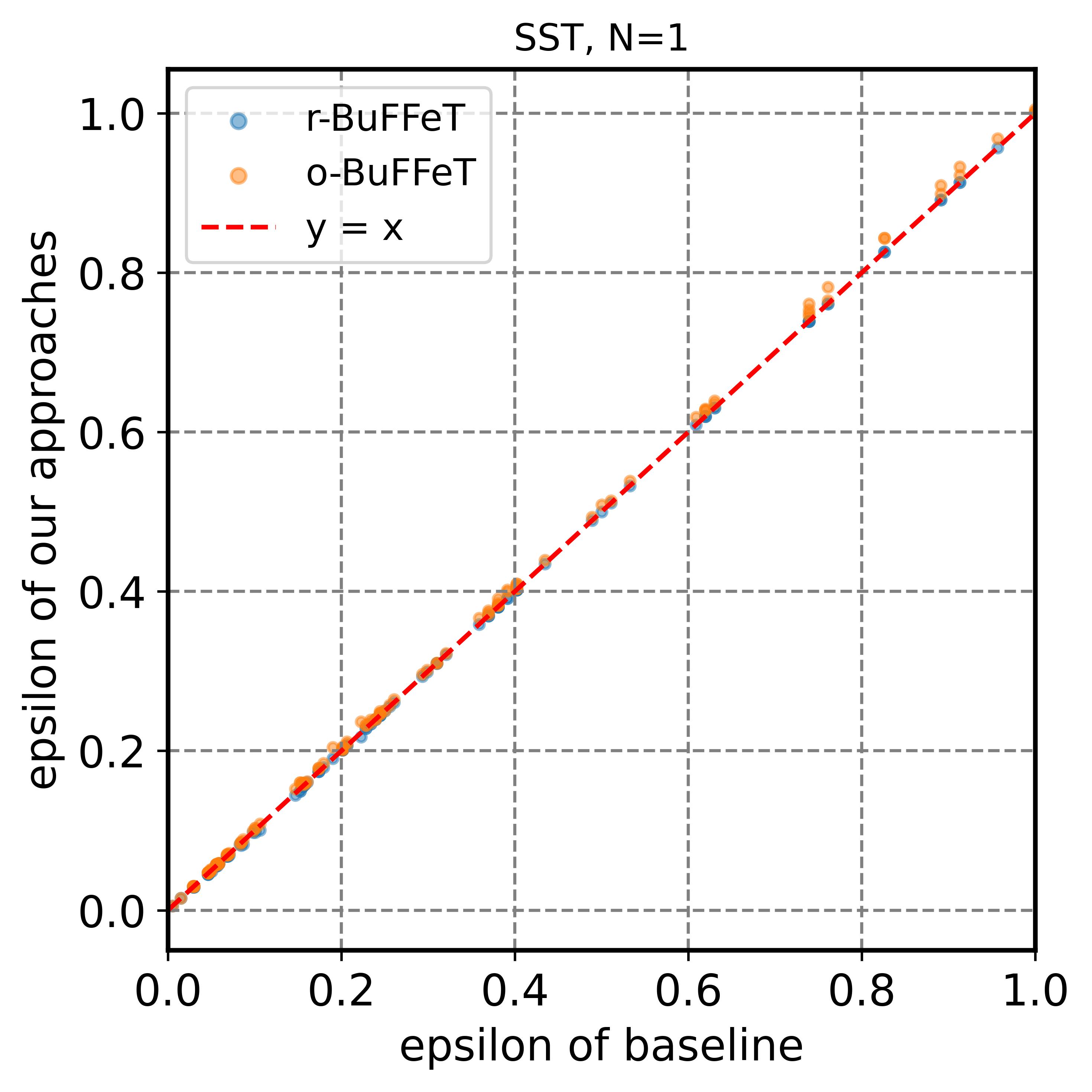}
\end{subfigure}
\quad
\begin{subfigure}[b]{0.3\textwidth}
\includegraphics[width=\textwidth, alt = {\sst, $\modelLayer = 2$}]{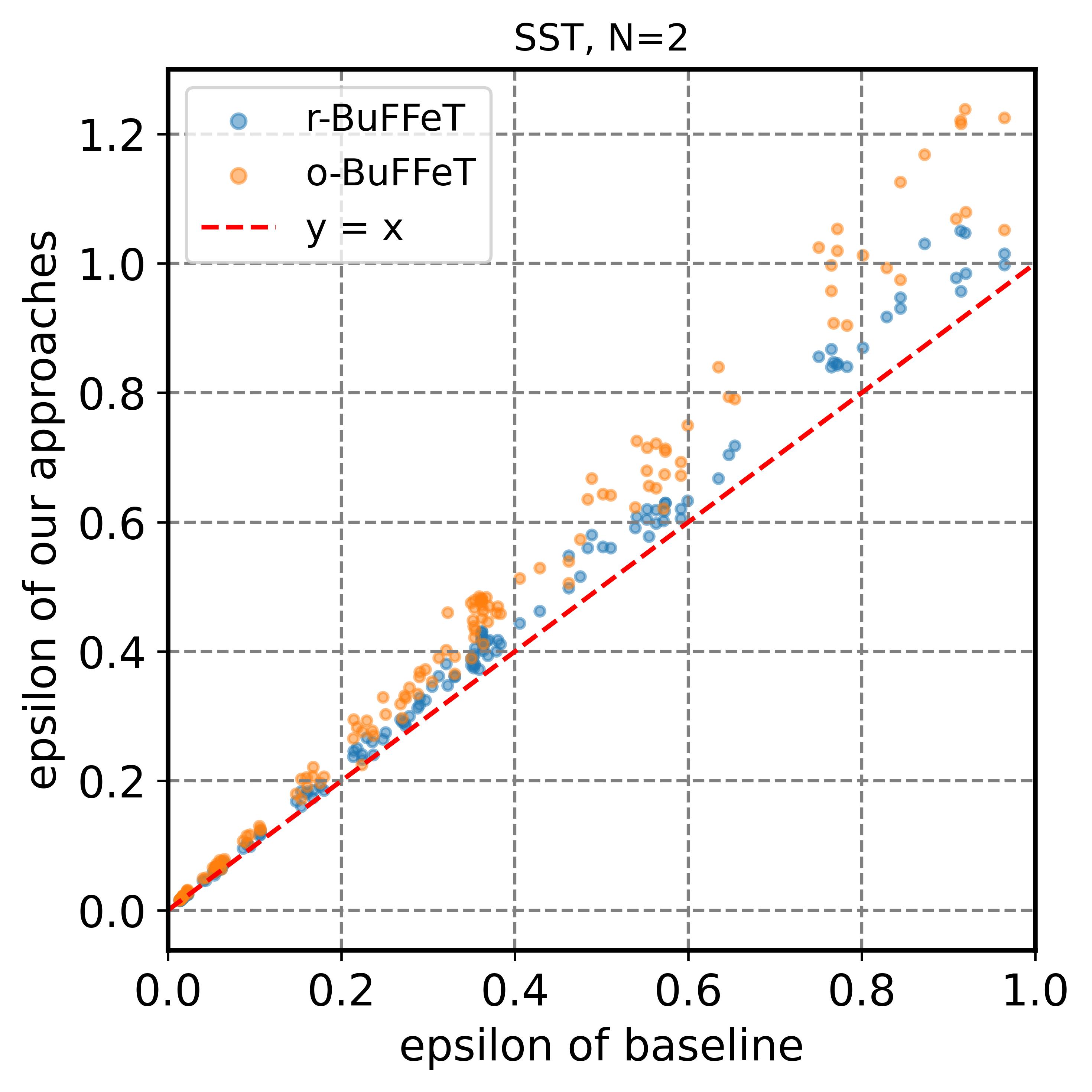}
\end{subfigure}
\quad
\begin{subfigure}[b]{0.3\textwidth}
\includegraphics[width=\textwidth, alt = {\sst, $\modelLayer = 3$}]{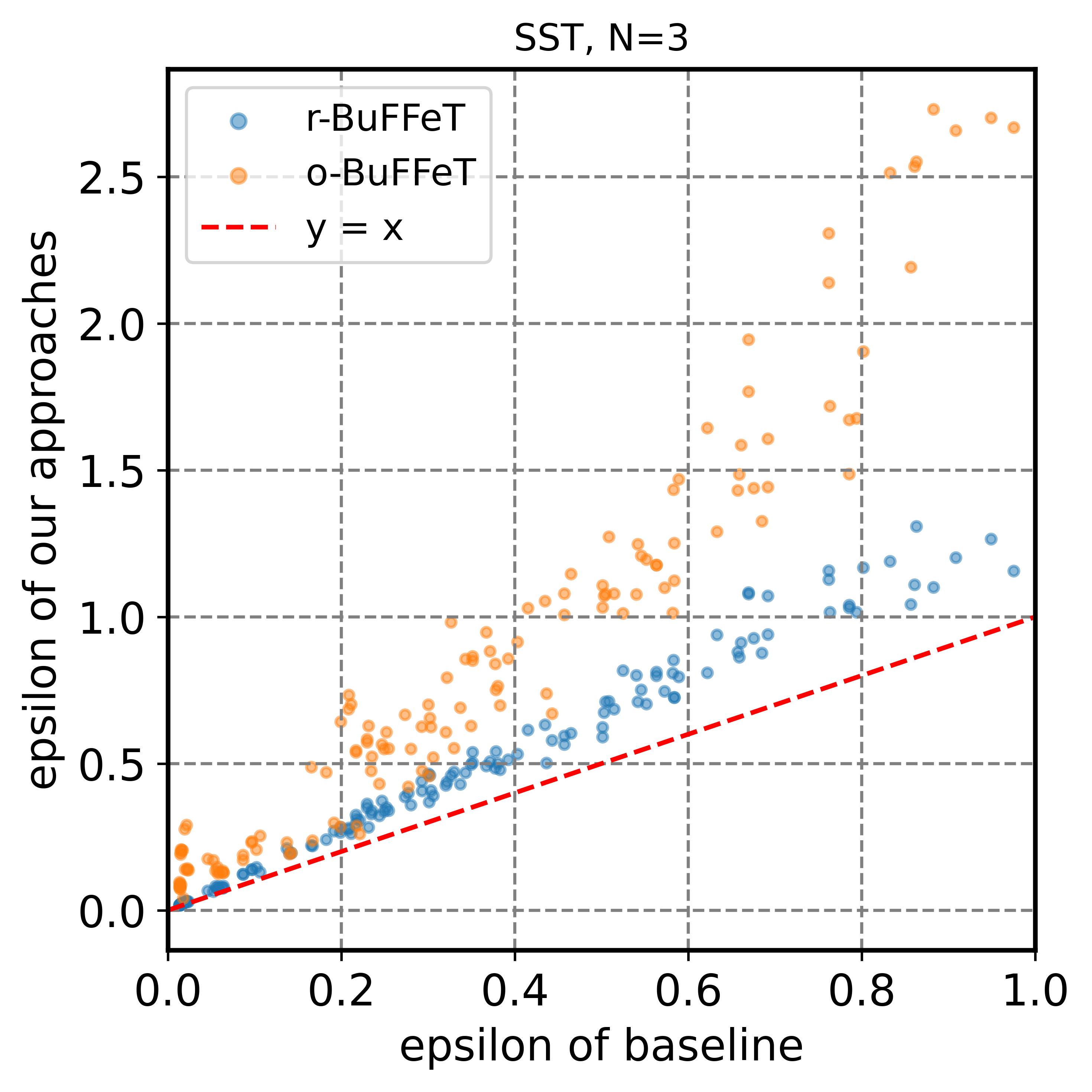}
\end{subfigure}
\begin{subfigure}[b]{0.3\textwidth}
\includegraphics[width=\textwidth, alt = {\sst, $\modelLayer = 6$}]{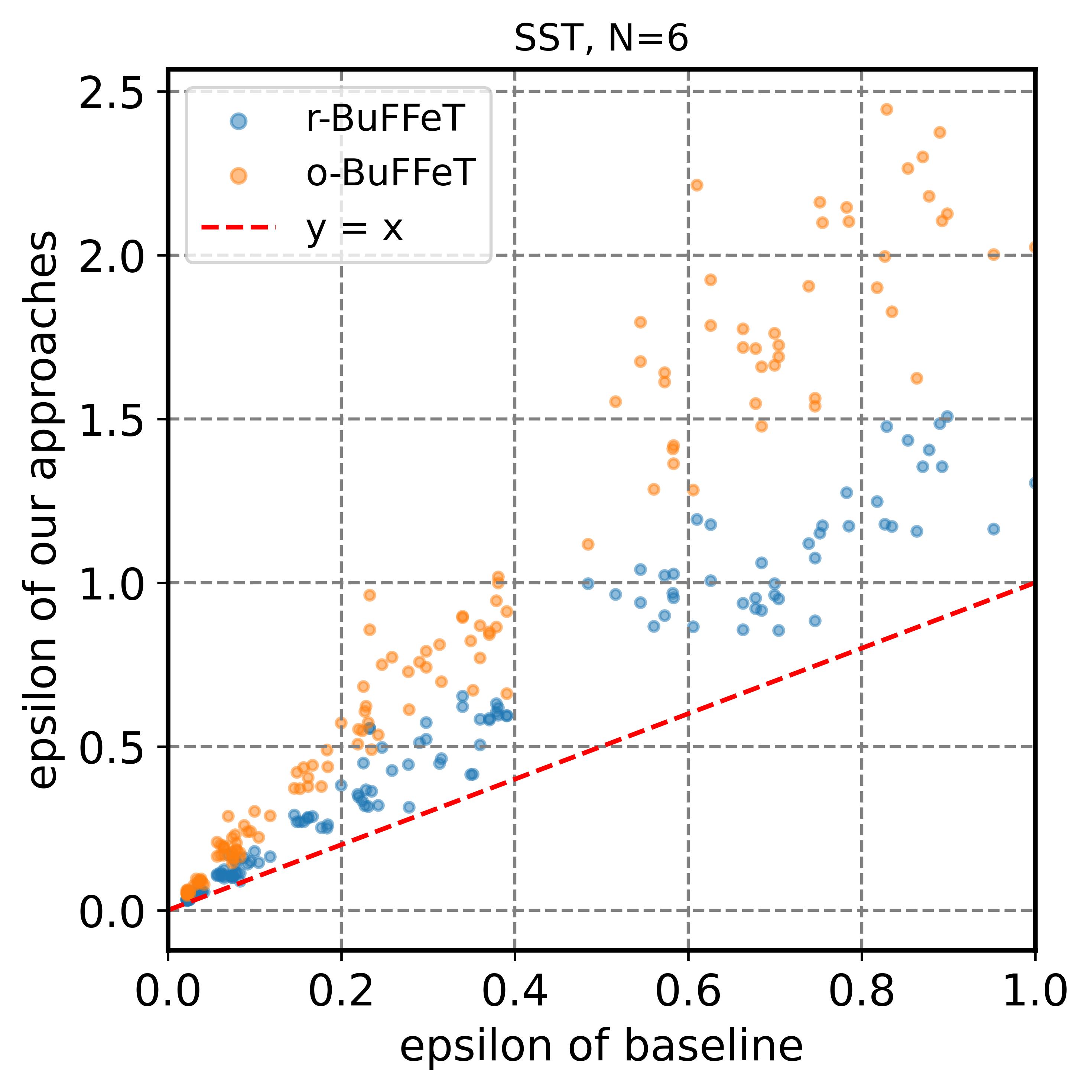}
\end{subfigure}
\quad
\begin{subfigure}[b]{0.3\textwidth}
\includegraphics[width=\textwidth, alt = {\sst, TinyBERT}]{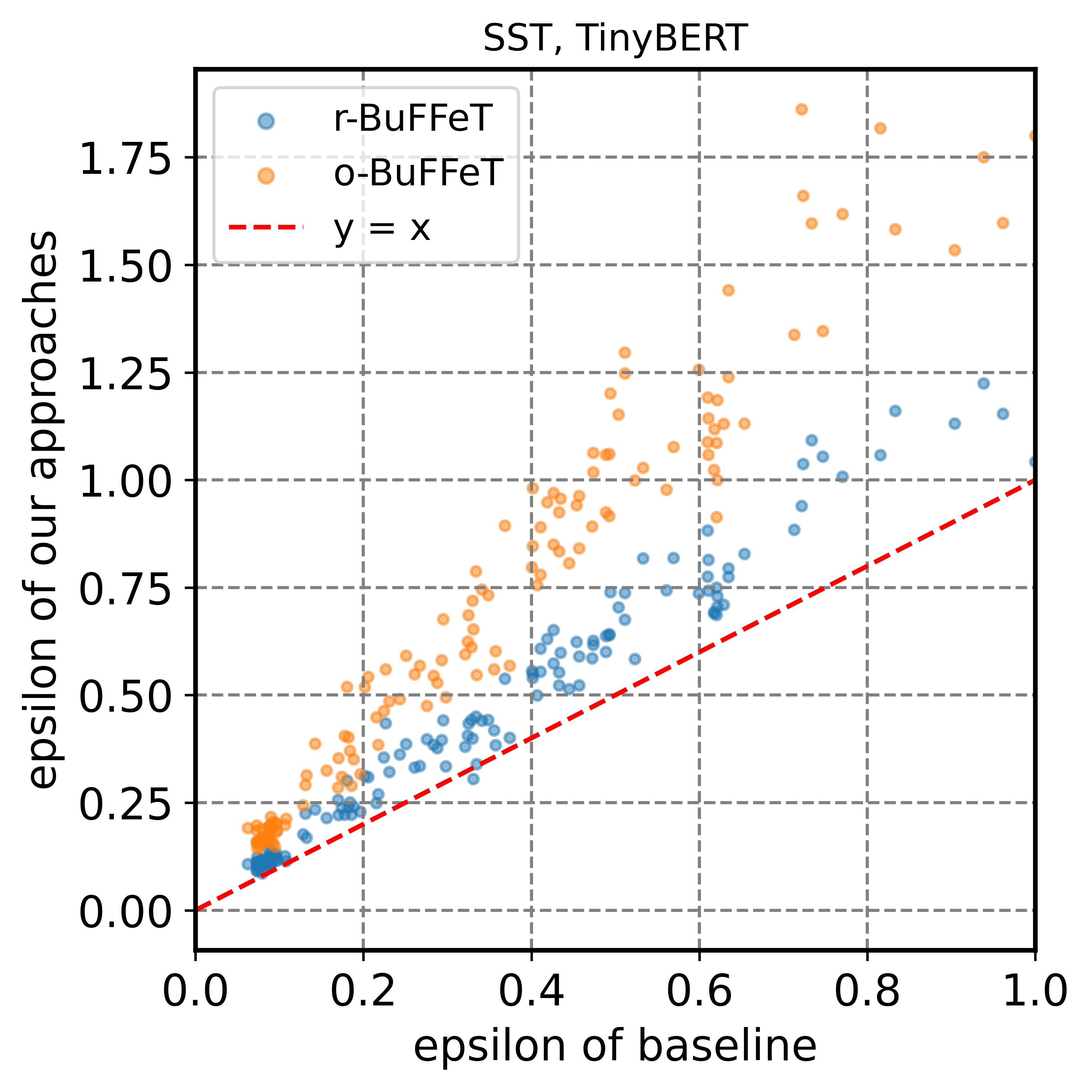}
\end{subfigure}

\begin{subfigure}[b]{0.3\textwidth}
\includegraphics[width=\textwidth, alt = {\yelp, $\modelLayer = 1$}]{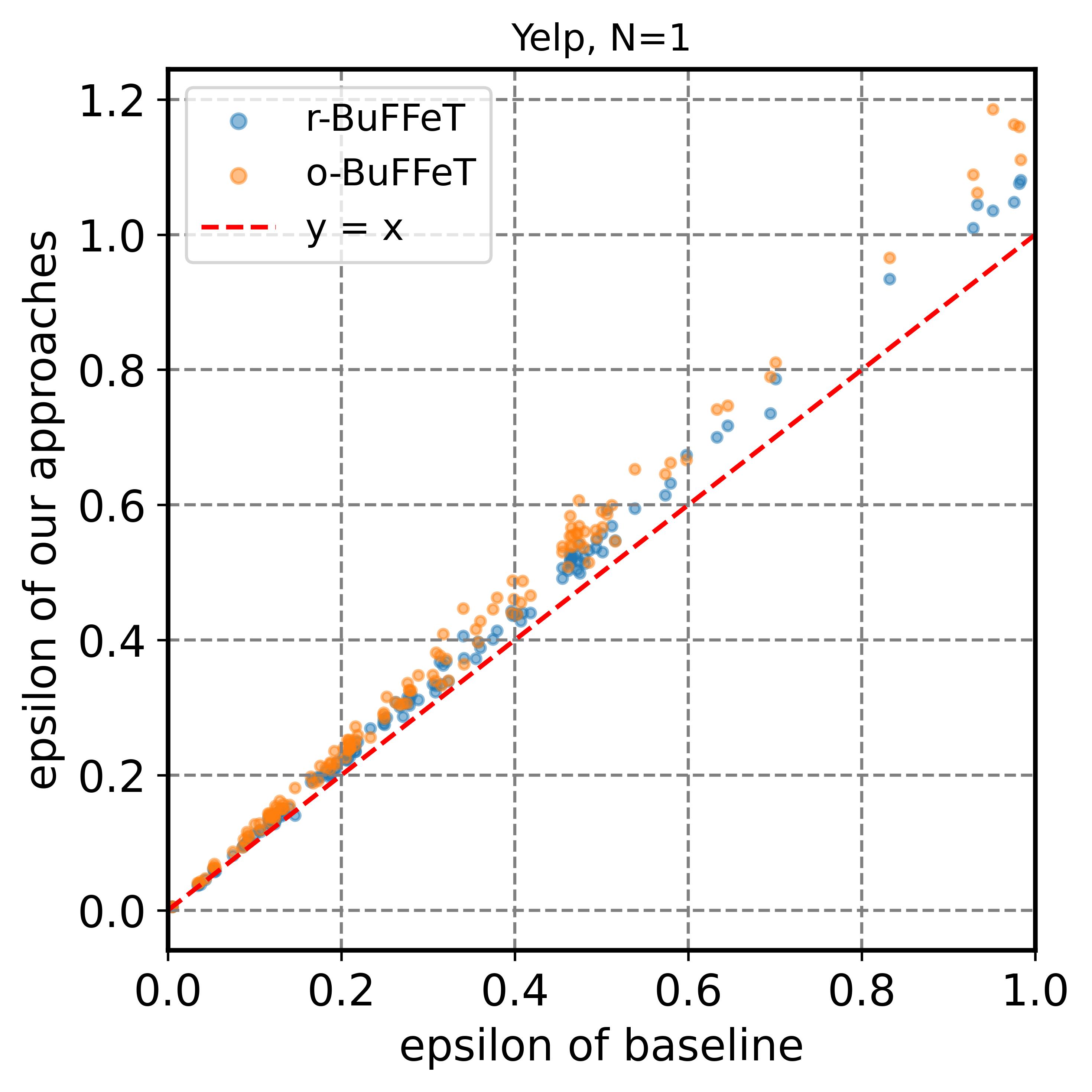}
\end{subfigure}
\quad
\begin{subfigure}[b]{0.3\textwidth}
\includegraphics[width=\textwidth, alt = {\yelp, $\modelLayer = 2$}]{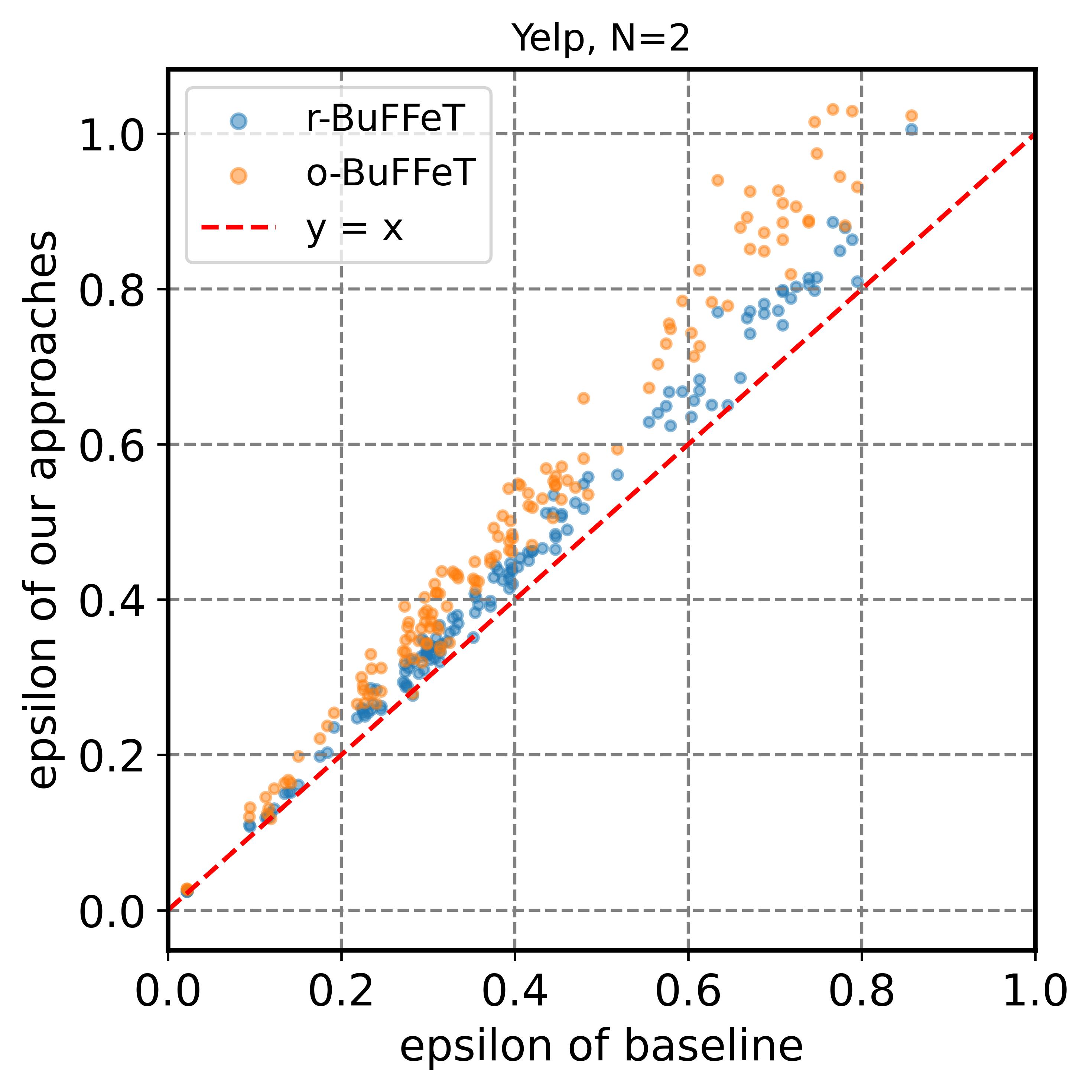}
\end{subfigure}
\quad
\begin{subfigure}[b]{0.3\textwidth}
\includegraphics[width=\textwidth, alt = {\yelp, $\modelLayer = 3$}]{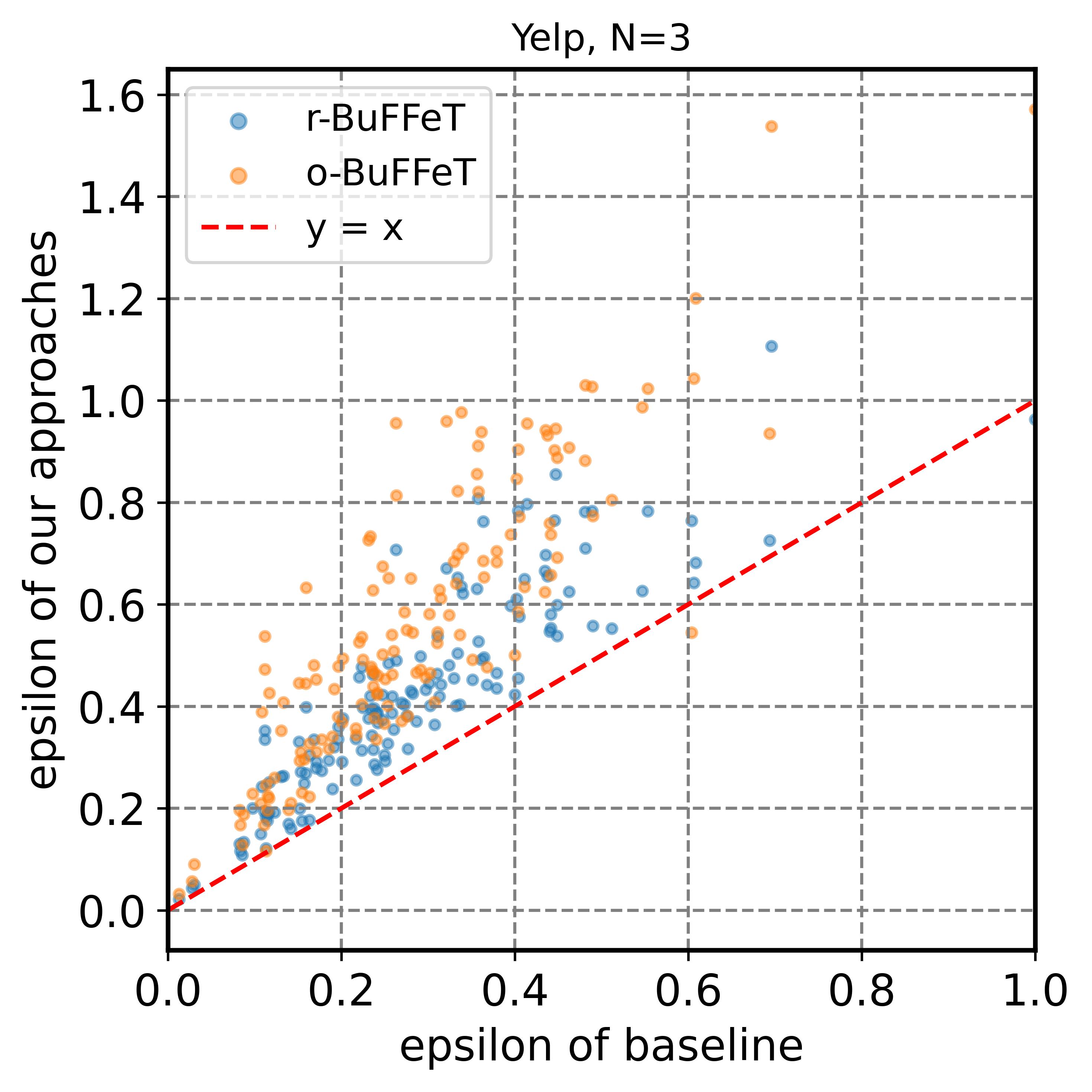}
\end{subfigure}
\begin{subfigure}[b]{0.3\textwidth}
\includegraphics[width=\textwidth, alt = {\yelp, $\modelLayer = 6$}]{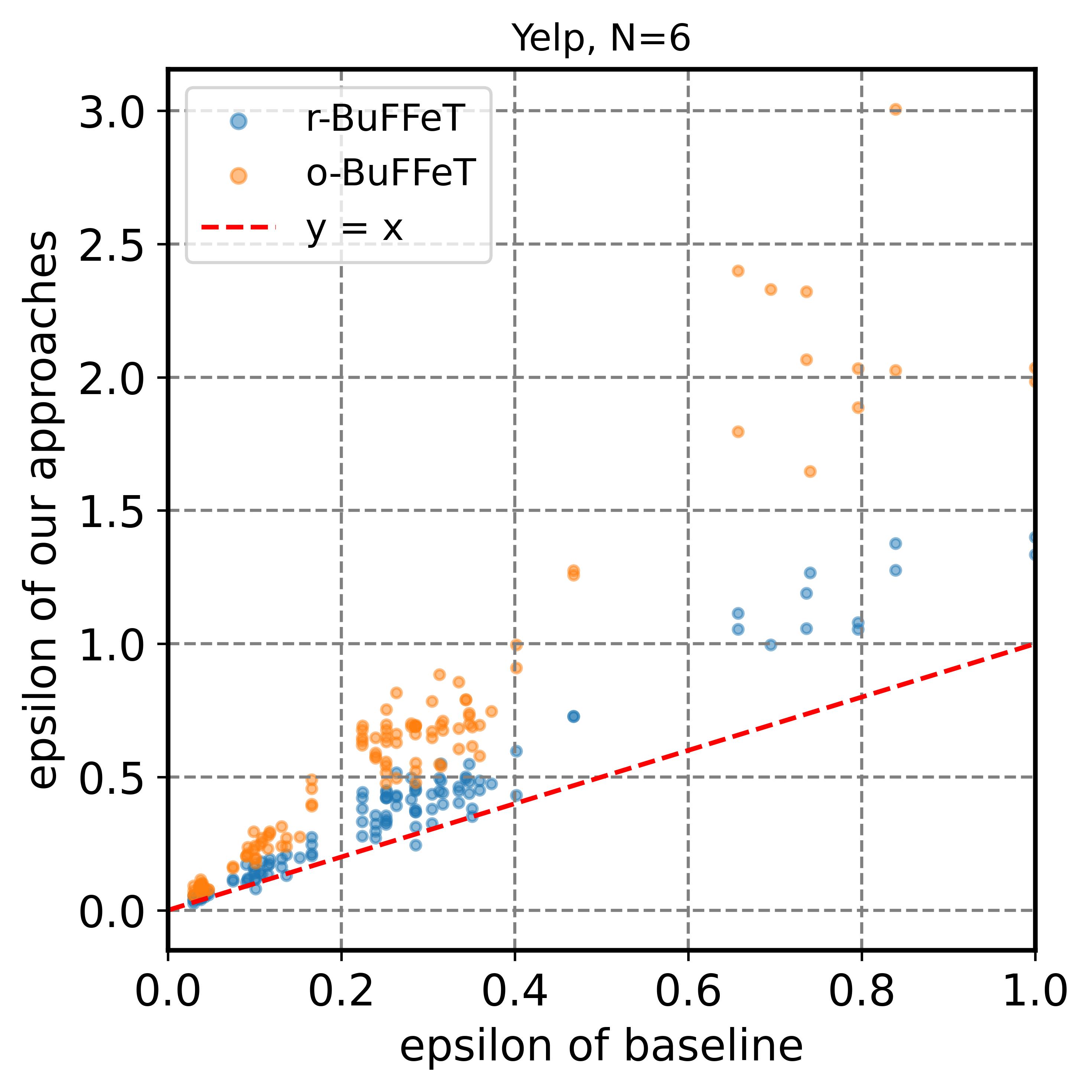}
\end{subfigure}
\quad
\begin{subfigure}[b]{0.3\textwidth}
\includegraphics[width=\textwidth, alt = {\yelp, TinyBERT}]{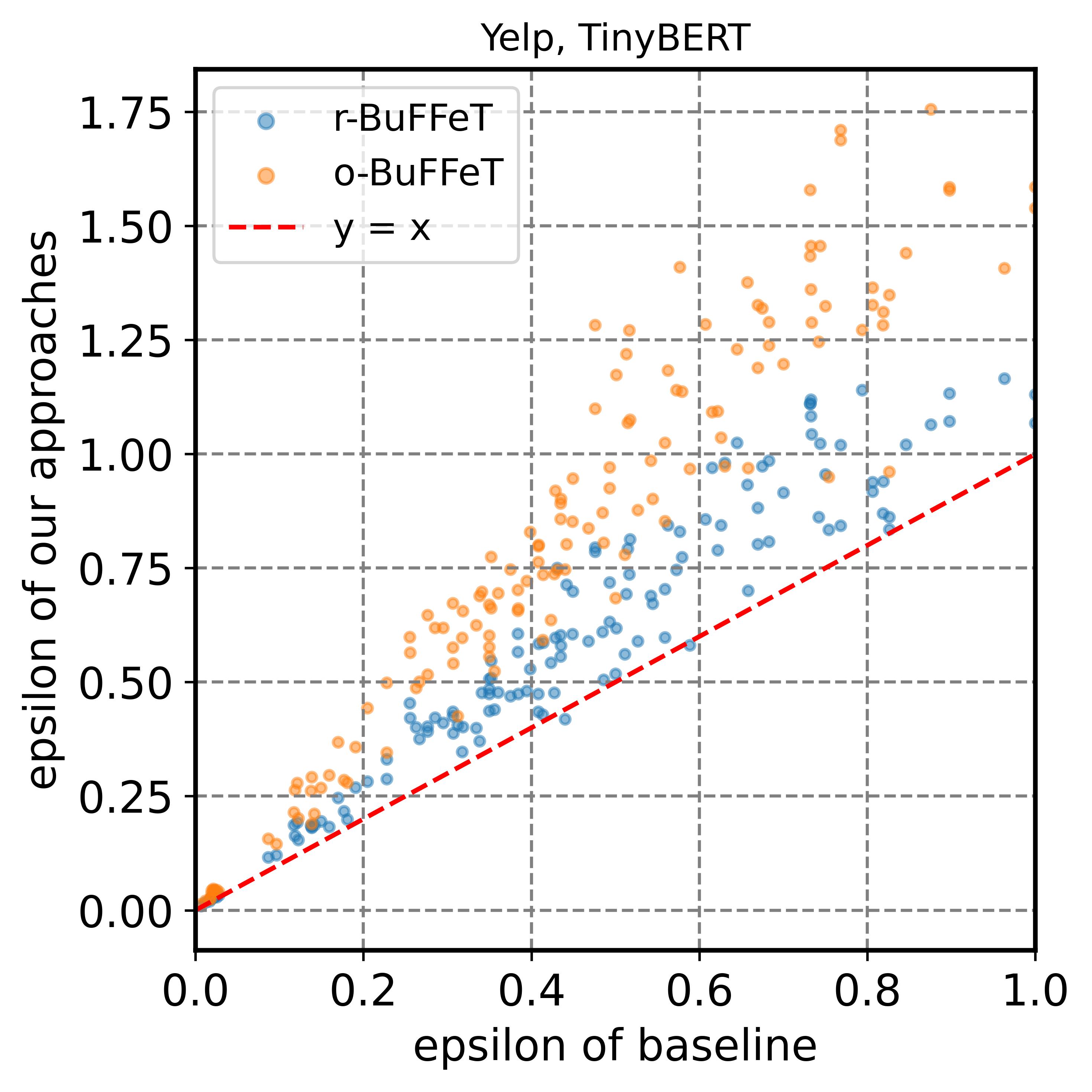}
\end{subfigure}
    \caption{Comparison with the baseline approach in terms of maximal verified $\epsilon$. The red line in each sub-figure, i.e., $y=x$, signifies the equivalent performance between different approaches, so the points on its top left are the cases our approaches outperform the baseline. Steeper points signify higher outperformance. 
    }
    \label{fig:scatterPlot}
\end{figure}

\begin{figure}[!tb]
    \centering

    \begin{minipage}[b]{0.48\textwidth}
        \centering
        \includegraphics[width=\textwidth, alt = {Box plot for \sst}]{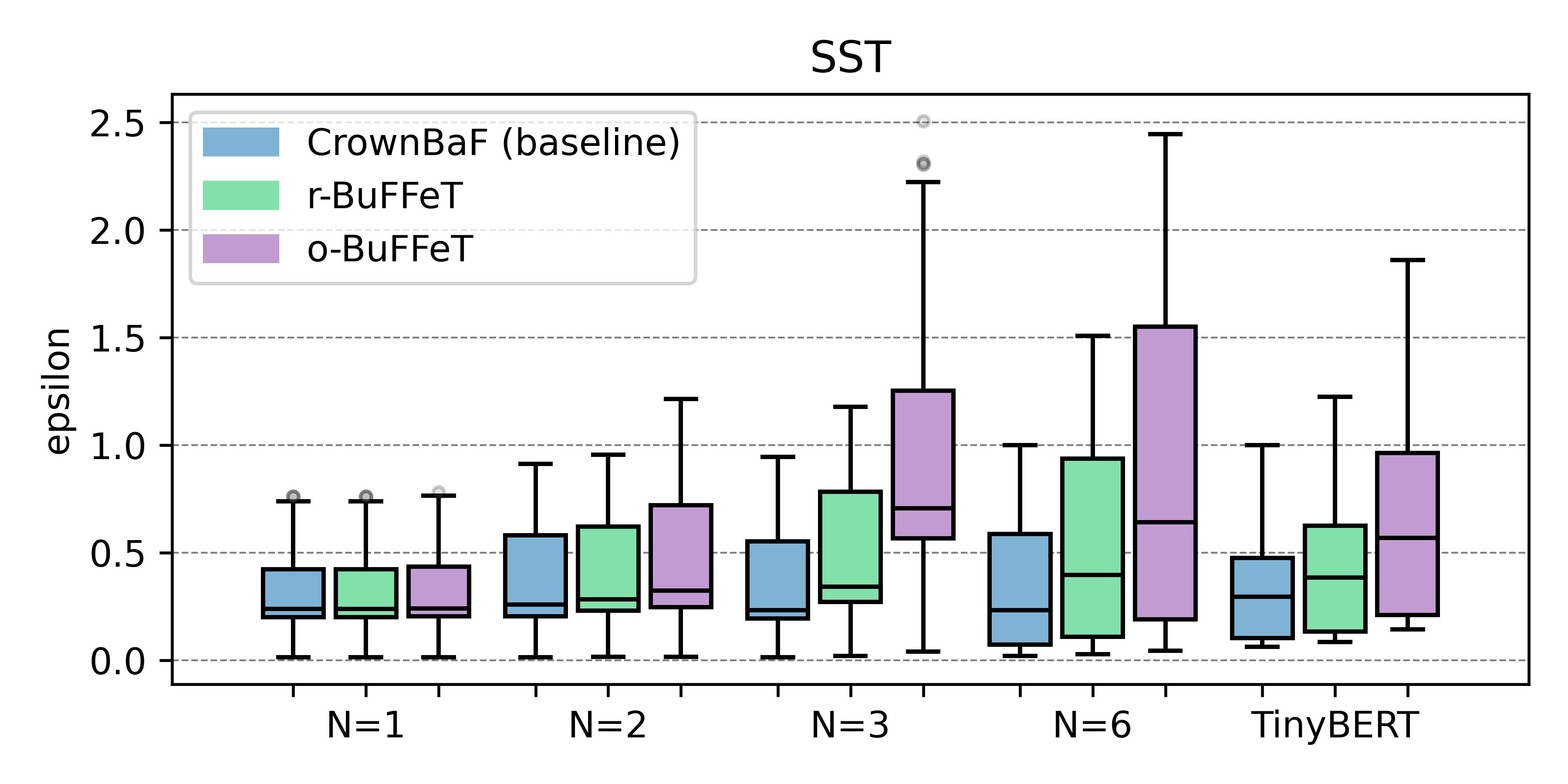}
    \end{minipage}
    \hfill
    \begin{minipage}[b]{0.48\textwidth}
        \centering
        \includegraphics[width=\textwidth, alt = {Box plot for \yelp}]{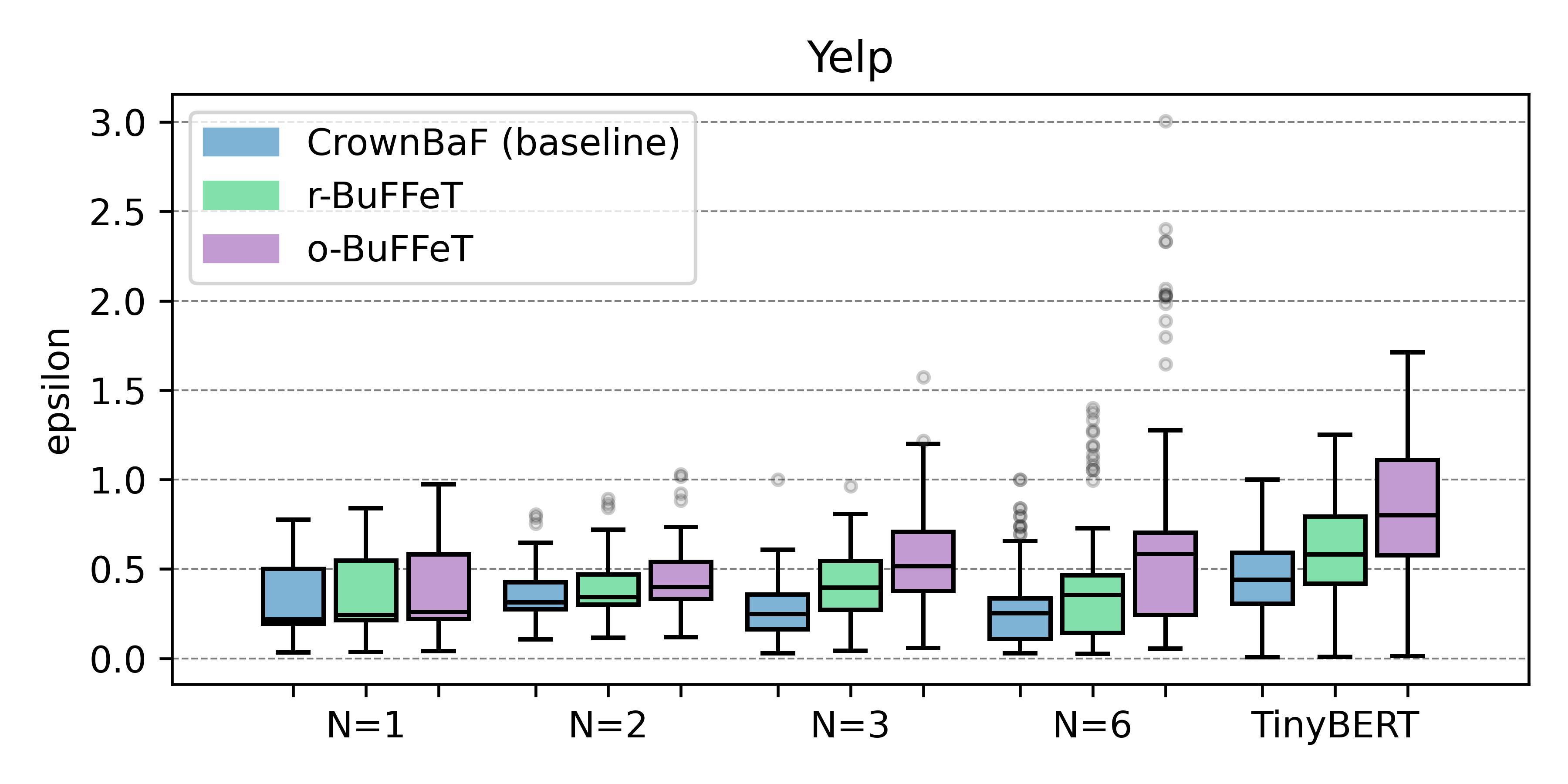}
    \end{minipage}
    \caption{The statistical information about the performances of the three approaches. The $y$-axis shows the maximal $\epsilon$ that can be verified by each approach. }
    \label{fig:boxPlot}
\end{figure}

In Fig.~\ref{fig:scatterPlot}, we present a comparison of \tool with the baseline approach. In each sub-figure, each point denotes a verification task; the $x$-axis denotes the maximal $\epsilon$ verified by the baseline, and the $y$-axis denotes the maximal $\epsilon$ verified by our approach. For better visualization, we have normalized the values of the maximal $\epsilon$ of the baseline approach to $[0, 1]$, by performing a linear rescaling. The red line (i.e., $y = x$) in each sub-figure suggests the performances of the baseline, namely, if a point is located above the the red line, it implies that our approach outperforms the baseline. The more separated a point is from the red line, the greater performance advantage  our approach achieves. Moreover, to show more rigorous statistics (e.g., medians, quartiles, min/max), we also present the box plots in Fig.~\ref{fig:boxPlot} to compare across different approaches. 

\smallskip\noindent\textit{Comparison with Baseline} By Fig.~\ref{fig:scatterPlot}, we observe that, for most of the verification tasks, both of our approaches outperform the baseline approach. In particular, as the number $\modelLayer$ of encoder layers grows, the outperformance becomes more evident. For instance, when $\modelLayer = 1$, both \toolp and \toola perform similarly with the baseline for \sst, and   marginally outperform the baseline for \yelp. When $\modelLayer = 6$, in both \sst and \yelp, the outperformance of \toolp can be as great as 1.8 times, and the outperformance of \toola can be 3.6 times. This observation also applies to TinyBERT, for which the outperformance can be as much as about 2.7 for both \sst and \yelp; as the TinyBERT we adopted has 4 layers, its performance advantage is similar to the standard transformer with $\modelLayer = 3$ or $6$. Such growth of outperformance w.r.t. the model depth can also be evidenced by the box plots in Fig.~\ref{fig:boxPlot}, which shows that the performances of different approaches are comparable when $\modelLayer = 1$, while our approaches outperform the baseline when $\modelLayer = 6$. This is reasonable, because as the number of encoder layers grows, our approaches can bring more approximation refinement, which leads to greater performance advantages. By such refinement, our approaches significantly improve the verification ability, enabling the certification of greater safety boundaries.    

\smallskip\noindent\textit{Comparison between \toolp and \toola} By Fig.~\ref{fig:scatterPlot}, we can also compare the performances between \toolp and \toola. While both of the approaches outperform the baseline, we can find that in most cases \toola performs better than \toolp, which is often able to certify greater perturbations. This can be shown by Fig.~\ref{fig:scatterPlot}, in which the points of \toola are mostly located above the points of \toolp, across all different benchmarks including TinyBERT. Moreover, the performance advantages also grow as model depths increase, signifying the strengths of \toola in handling complex transformer models. This observation is also evidenced by the box plots in Fig.~\ref{fig:boxPlot}, in which we find that the medium of \toola outperforms \toolp stably, and the maximum of \toola outperforms \toolp significantly, as the model depths grow. 

\smallskip\noindent\textit{Underperformance of \tool} 
By Fig.~\ref{fig:scatterPlot}, we also observe some cases when \tool underperforms the baseline.
\begin{compactitem}
    \item {\it For \toolp}: There are a few cases when \toolp underperforms the baseline, e.g., $\modelLayer = 6$ for \yelp, and TinyBERT for both \sst and \yelp. This is expected, because the selection strategy of \toolp is essentially a heuristic rule that selects $\alpha$ based on the area of approximation~\cite{singh2019abstract}, but existing study~\cite{zhang2022provably} has shown that this strategy may not always work. 
    \item {\it For \toola}: We observe one instance (\yelp, \modelLayer=3) when \toola underperforms the baseline, which demonstrates that the iterative optimization can still fail to find a solution better than the baseline. It's worth noting that, despite the large-scale experiments, we only observe the occurrence of one such instance, which demonstrates the empirical effectiveness of \toola.
\end{compactitem}

\researchquestion{Is changing $\alpha$ useful to refine the approximation?}
\begin{figure}[!tb]
    \centering
    \begin{minipage}[b]{0.48\textwidth}
        \centering
        \includegraphics[width=\textwidth, alt = {change of \margin Example 1}]{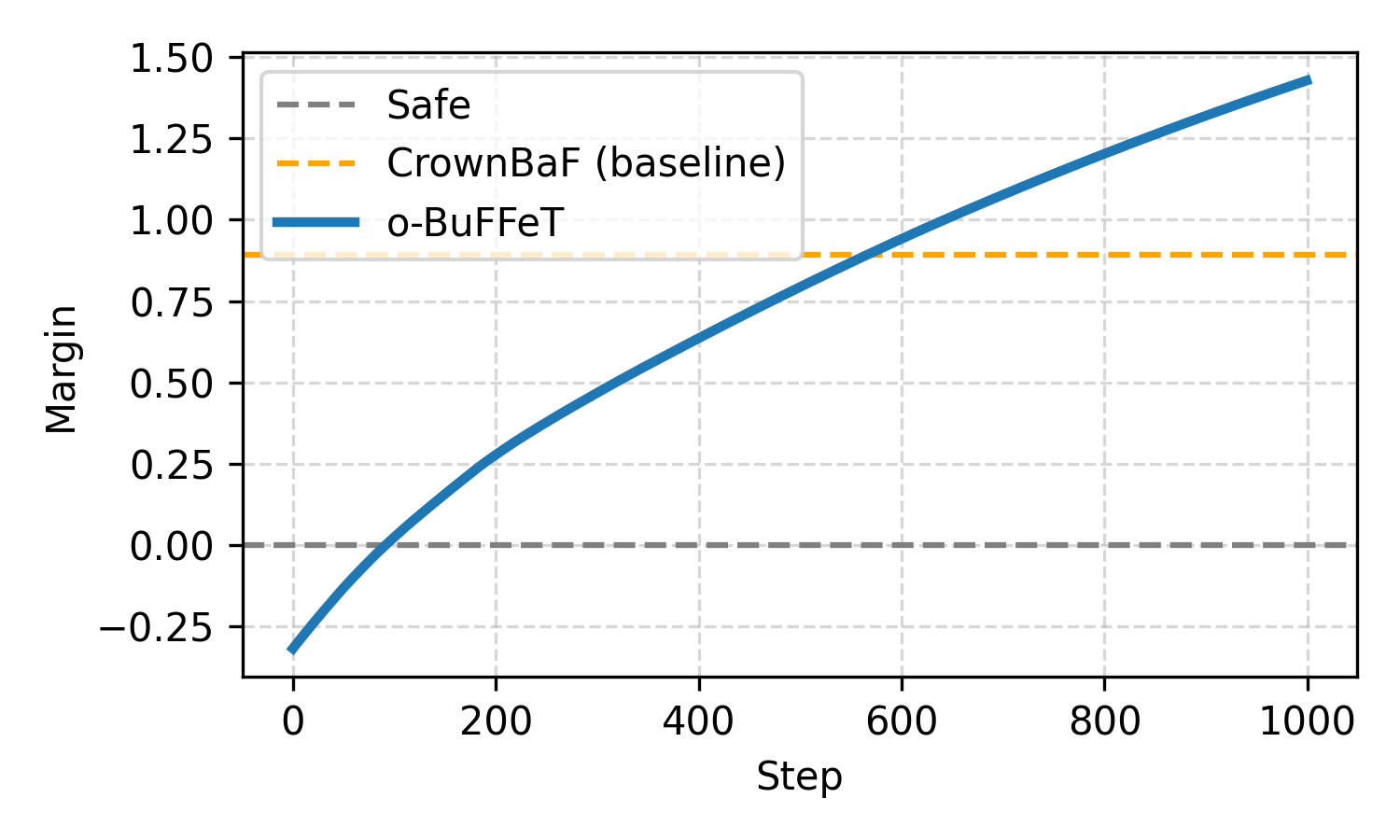}
    \end{minipage}
    \hfill
    \begin{minipage}[b]{0.48\textwidth}
        \centering
        \includegraphics[width=\textwidth, alt = {change of \margin Example 2}]{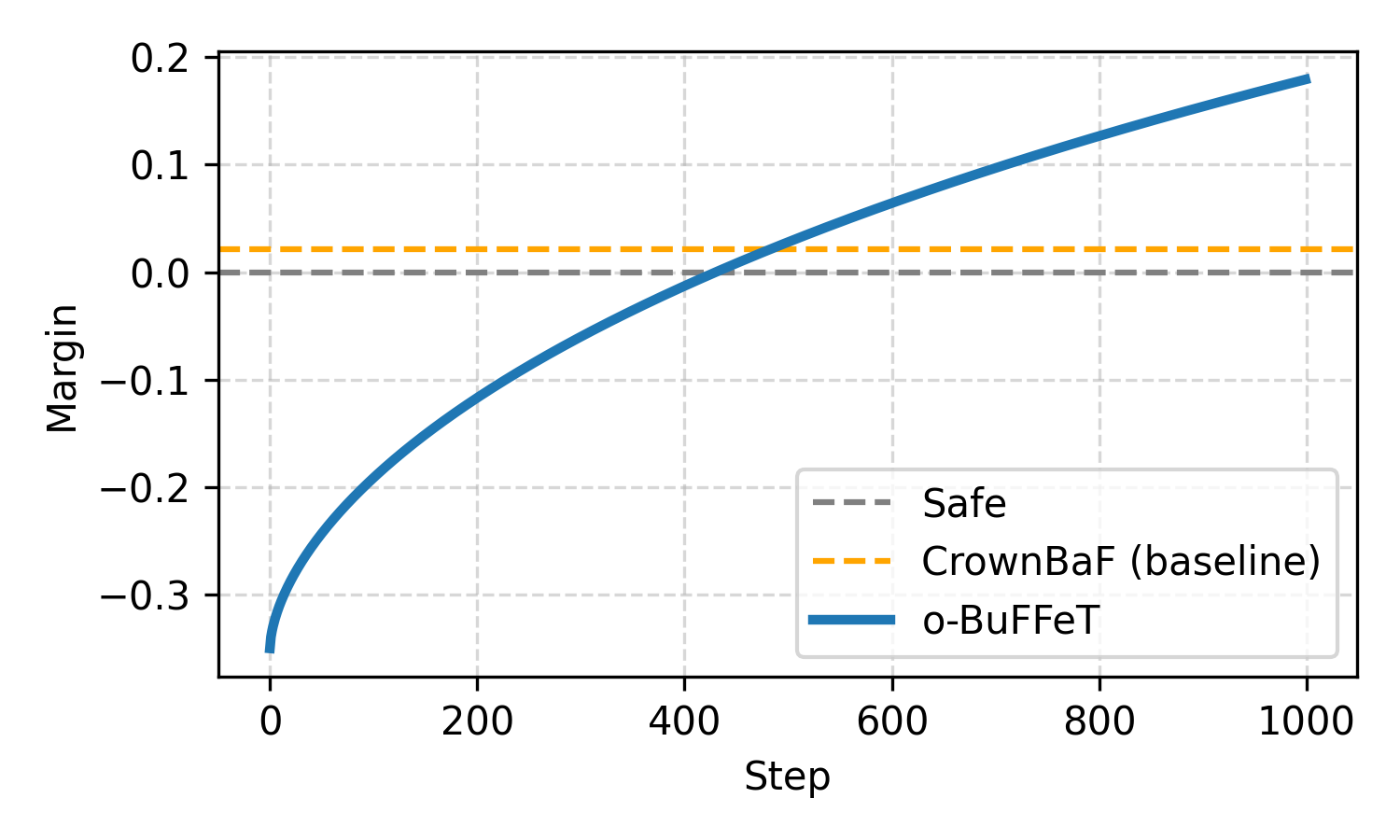}
    \end{minipage}
    
    \begin{minipage}[b]{0.48\textwidth}
        \centering
        \includegraphics[width=\textwidth, alt = {change of \margin Example 3}]{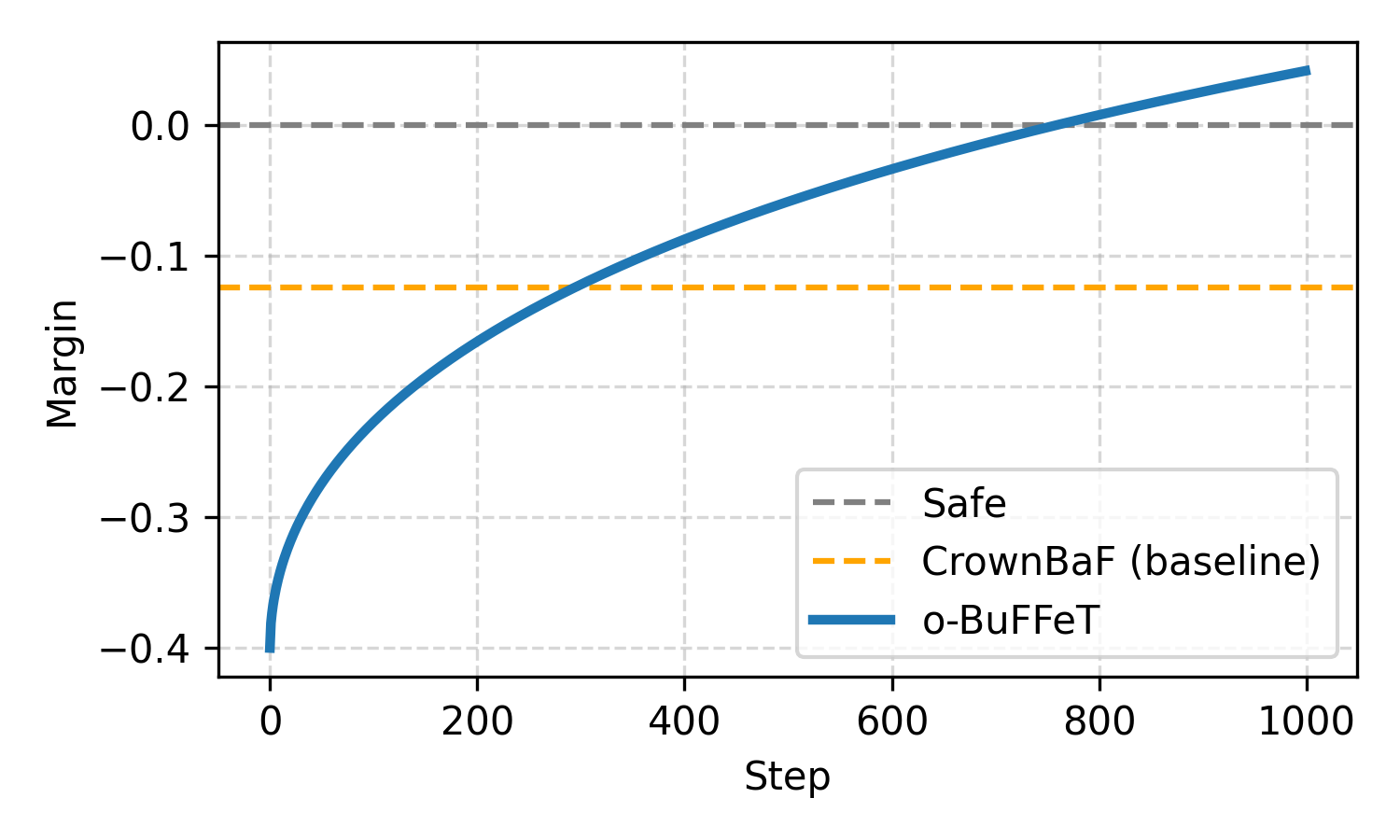}
    \end{minipage}
    \caption{The change of \margin over the optimization process of \toola}
    \label{fig:change}
\end{figure}

In Fig.~\ref{fig:change}, we show how \margin (see Eq.~\ref{eq:margin}, which implies the degree of robustness satisfaction in a verification problem) changes throughout a verification process with \toola (i.e., the goal here is to verify a given $\epsilon$, rather than finding maximal $\epsilon$ as done in RQ1). Essentially, over the optimization process, the value of \margin should keep growing, and once it becomes positive, it signifies that the problem is verified.  Specifically, we select three verification problems from \yelp and rerun the verification using \toola to demonstrate this process. In each plot, we mark the line of $0$. Normally, \toola will be terminated as soon as $\margin$ crosses this line; for presentation, we do not terminate \toola at that point but we always terminate \toola after 1000 iterations. Moreover, we also mark the \margin of the baseline as a reference. 

By Fig.~\ref{fig:change}, we observe the impact of changing $\alpha$ to the final verification results, as well as the usefulness of the optimization strategy in \toola. In particular, in the first plot, \toola starts with a randomly initialized $\alpha$, which introduces overly imprecise approximation and thus cannot lead to verification of the given $\epsilon$. Then, \toola optimizes $\alpha$ to refine the approximation, and as it proceeds, the \margin value crosses the $0$ line with around 100 steps, hence verifying the given $\epsilon$. As the optimization process goes on, the approximation gets further refined and the \margin keeps growing. At around 600 step, \toola crosses the \margin of the baseline, and still continues growing after that. This observation suggests that, the optimization strategy allows \toola to deal with verification tasks with different levels of difficulties; in particular for difficult problems, \toola offers the possibility of solving them if more time budget could be given.

In the third plot, the \margin given by the baseline is negative, signifying that it cannot solve the problem. In contrast, \toola keeps refining the approximation and finally the \margin crosses $0$ with around 800 steps. This example demonstrates a more useful usage scenario of \toola: in case the fixed verification results given by the baseline approach are not sufficient to solve the problem, 
the optimization process of \toola offers the ability to iteratively refine the approximation until the verification problem is solved. For complex transformer models and difficult verification problems, \toola can be much more useful to eliminate false alarms and deliver reliable verification results.

\researchquestion{How efficient are \toolp and \toola?}

In Table~\ref{table:timeComparison}, we compare our approach with the baseline in terms of time costs. For each model, we report the average time costs of our approaches and the baseline approaches, over all the verification problems with each dataset.

\begin{table}[!tb]
\centering
\caption{Comparison of different approaches in terms of time costs. $\Delta$ denotes the ratio of the time costs of our approach to that of the baseline.
} 
\label{table:timeComparison}
\setlength{\tabcolsep}{6pt}
			\begin{tabular}{lcr|rr|rr}
				\toprule
				\multirow{2}*{Datasets} & \multirow{2}*{\modelLayer} &  \multicolumn{5}{c}{time costs (secs)}\\
                \cline{3-7}
                ~                     &                  ~&  baseline  &    \toolp & $\Delta$    & \toola & $\Delta$ \\
                
				\midrule
                				\multirow{5}*{\sst}     & 1 & 0.06                          & 0.09 &  1.5                     & 5.77  & 96.2              \\
				~                       & 2 & 0.17                           &  0.26   & 1.5                    & 14.70     &     86.5            \\
				~                       & 3 & 1.05                           & 1.76  & 1.7                     & 101.72   & 96.9           \\ 
                ~                       & 6 & 11.38                           & 12.57 & 1.1                      & 973.50    & 85.5            \\ 
                \cline{2-7}
                ~                       & {\scriptsize{TinyBERT}} & 4.30                          & 5.27  & 1.2                     & 281.07    &   65.4          \\ 
                \midrule
                \multirow{5}*{\yelp}    & 1 & 0.08                          & 0.13 & 1.6                    & 2.68      &  33.5          \\
				~                       & 2 & 0.17                        & 0.23   & 1.4                   & 12.97    &    76.3          \\
				~                       &3 & 0.92                       & 1.26    & 1.4                   & 77.61       &   84.4        \\
                ~                       & 6 & 8.05                           & 9.43   & 1.2                    & 673.42   &   83.7          \\ 
                \cline{2-7}
                ~                       & {\scriptsize{TinyBERT}} & 2.29                          & 3.34  & 1.5                     & 175.38     &     76.6     \\ 
				\bottomrule
			\end{tabular}
\end{table}
By the results, first we observe that, the efficiency of \toolp is comparable with the baseline approach, but is slightly worse than the baseline in average, as evidenced by the ratio of the time cost, i.e., between 1.1 and 1.7.
This is expected, because while the baseline takes a fixed policy to compute linear bounds, \toolp spends additional time on computing the ranges of ReLU inputs (as described in~\S{}\ref{sec:ruleBased}) to derive more precise bounds.  
Nevertheless, given the improvement of \toolp in certified $\epsilon$, this trade-off is acceptable in practical scenarios, and moreover, the efficiency compromise is not significant and the time cost of \toolp is in the same order of magnitude with the baseline. 

Moreover, we can see that \toola introduces more evident overhead. Compared to the baseline, it costs 33.5 to 96.9 times more to find the maximal $\epsilon$. These costs can be attributed to the iterative optimization process; since each iteration of the optimization essentially consists in a linear propagation that costs roughly the same as the baseline, it indicates that \toola in average takes these numbers of iterations for optimization, leading to the additional time costs.

However, we also have the following observations regarding the results:
\begin{inparaenum}[1)]
    \item 
    verification for 6-layer transformers and TinyBERT, i.e., the models that have been widely deployed in practice, can be finished in minutes. Given the significant precision advantage of \toola, the performance overheads could be acceptable in many use cases; 
    \item despite the high time costs, the ratio of \toola to the baseline is relatively stable (roughly two orders of magnitude for all the models) and does not increase with the growth of $\modelLayer$. As explained previously, this can be attributed to the iterative search by optimization, which takes a (relatively) stable number of iterations to solve the problems.  Given the high efficiency of the baseline, it suggests that \toola should also be scalable for more complex models;
    \item as the time costs are closely relevant to the number of iterations for optimization, the optimizer selection and configuration, as well as the hardware environments for optimization, could matter. To this end, the efficiency of \toola could be further strengthened in practice, given better environments or domain expertise. 
\end{inparaenum}

\section{Related Work}
\smallskip\noindent\textit{Classic Neural Network Verification}
Verification for feed-forward networks has been studied widely; overviews can be found in recent surveys~\cite{liu2021algorithms,Konig2024Critically,Boudardara2024Review} and the competition reports~\cite{brix2023first,kaulen20256th,brix2024fifth}. These works can be split into two families: \emph{complete} methods and \emph{incomplete} methods. Complete methods aim to prove a property by exhaustive exploration of state space, often with SMT~\cite{katz2017reluplex}, MILP~\cite{Tjeng2019MILP} or branch and bound~\cite{bunel2020branch};  they are more expensive but precise. Incomplete methods use abstract interpretation~\cite{gehr2018ai2,singh2019abstract,zhang2018efficient} to build convex bounds layer by layer, which scales better to large models.
Typical bounding methods include DeepPoly~\cite{singh2019abstract}, CROWN~\cite{zhang2018efficient}, which construct linear upper/lower bounds for nonlinear layers (e.g., ReLU) and propagate them through the network.
Refinement on top of these relaxations (e.g., via branch-and-bound~\cite{bunel2020branch,shi2024neural}) further tightens bounds when needed, achieving both efficiency and completeness. Our approaches are in line with them, but we target transformer verification that remains underexplored.

\smallskip\noindent\textit{Verification of Transformers}
Compared with classic neural network verification, transformer verification is new and more challenging because self-attention introduces bilinear terms and softmax.
Existing works extend abstract-interpretation-style relaxations to attention. Shi et al.~\cite{Shi2020Robustness} give a foundational framework that designs linear bounds for transformer components, including a planar relaxation for dot products in self-attention, and then propagates these bounds end-to-end.
Later, Bonaert et al.~\cite{bonaert2021fast} study faster and more precise certification techniques for transformers, which introduces a new multi-norm zonotope abstract domain. 
Another study~\cite{liao2023transformers} considers \emph{sparsemax}-based transformers and reduces robustness verification to a Mixed Integer Quadratically Constrained Programming (MIQCP) problem. This enables exact reasoning but is less scalable.
Zhang et al.~\cite{zhang2024galileo} propose a general linear relaxation for multi-dimensional functions such as softmax. Meanwhile, Wei et al.~\cite{pmlr-v206-wei23c} give convex lower bounds and concave upper bounds for the softmax that fit convex optimization formulations.

Our view is that attention is the core mechanism and dot-products appear everywhere, bringing significant impact to verification results. While the planar dot-product bound of~\cite{Shi2020Robustness} is promising, it may be not sufficiently tight to handle real-world models. To this end, we devise a more precise construction that can significantly reduce over-approximation, leading to tighter certified bounds.

There is also a line of work~\cite{Casadio2023ANTONIO,Casadio2026NLP} that considers the meaningfulness of verifying  transformers in the context of NLP. Normal NLP pipelines involve an embedding function that encodes raw textual words into vectors, so an $\epsilon$-ball input region of a given textual word in embedding space may contain few words that are practically meaningful. Our problem definition in~\S{}\ref{sec:problemStatement} also suffers from this problem, because we consider perturbations to inputs in embedding space but bypass the embedding function; while this brings us the generality to handle continuous input modalities in domains, such as audio processing and computer vision, its direct application to local robustness verification of NLP models may lead to limited practical implication. Nevertheless, various techniques that can handle this issue, have been proposed in literature~\cite{Casadio2023ANTONIO,Casadio2026NLP}, e.g., using a ``hyperrectangle'' to encompass meaningful words; in practice, our approach can be adopted collaboratively with those techniques, and serve as an off-the-shelf verifier in the general pipeline proposed in~\cite{Casadio2026NLP}, with improved precision compared to existing verifiers.

\section{Conclusion and Future Work}
\label{sec:conclusion}
We present an abstract refinement approach for verification of transformers, inspired by linear relaxation of ReLU. We represent the aggregated bound by ReLUs and exploit two strategies to achieve refined linear bounds. \toolp is rule-based, which decides ReLUs' bound based on the input ranges of ReLUs. \toola formalizes the problem as an optimization problem, and calls solvers to iteratively explore the optimal bounds. Our evaluation results demonstrate the strengths of our approach in verification precision.

As future work, we consider two directions:
\begin{inparaenum}[1)]
    \item we can combine the different strategies to  improve efficiency, e.g., initialize $\alpha$ in the optimization of \toola by using the $\alpha$ from the baseline;
    \item we also plan to evaluate our approaches on transformers in other domains, e.g., transformer-based controllers~\cite{katz2022verification}.
\end{inparaenum}

\begin{credits}
\subsubsection{\ackname}
We thank the anonymous reviewers for their valuable comments. H. Liu is supported by JST BOOST Grant No. JPMJBS2406; Z. Zhang is supported by JST BOOST Grant No. JPMJBY24D7 and JSPS Grant No. JP25K21179; J. Zhao is supported by JSPS Grant No. JP24K02920.

\subsubsection{\discintname}
The authors have no competing interests to declare that are relevant to the content of this article.

\subsubsection{Data Availability Statement.}
All relevant data that support the findings of this paper are available in Zenodo~\cite{liu_2026_artifact}.
\end{credits}

\bibliographystyle{splncs04}
\bibliography{references}


\clearpage
\appendix
\section{A Detailed Introduction to Transformers}\label{sec:detailedTransformer}

In our experiments, we use a simplified Transformer-based architecture tailored for sentence-level sentiment classification. The overall model structure is illustrated in Figure~\ref{fig:architecture}. It follows the standard encoder-only Transformer design, consisting of a stack of encoder layers, followed by a pooling operation and a classifier head. This allows the model to dynamically weight contextual information for each token based on its similarity to other tokens in the sequence.

Each input sentence is first tokenized and mapped to a sequence of word embeddings. These embeddings are processed by multiple Transformer encoder layers. Each encoder layer contains two primary components: (1) a multi-head self-attention mechanism (i.e., the function $\fattention$ in Eq.~\ref{eq:transformer}), and (2) a feed-forward neural network (i.e., the function $\ffnn$ in Eq.~\ref{eq:transformer}), each followed by residual connections and layer normalization, as introduced in  Eq.~\ref{eq:transformer}.

After the final encoder layer, a pooling operation is applied—typically mean or [CLS] token pooling—to obtain a fixed-size representation of the input. This vector is passed through a linear classifier to produce the final sentiment prediction (e.g., positive or negative).

The internal structure of the self-attention mechanism is depicted in Figure~\ref{fig:architecture}. For a given input sequence, the model computes query ($\Qmat$), key ($\Kmat$), and value ($\Vmat$) matrices via learned linear projections. The attention score is computed using the scaled dot-product between $\Qmat$ and $\Kmat$ (Eq.~\ref{eq:attention}).

\section{A Detailed Introduction to~\cite{Shi2020Robustness}}\label{sec:detailedICLR}
In CROWN-like verification frameworks, robustness is typically analyzed by computing global output bounds given input perturbations. Specifically, for each embedding dimension, the corresponding lower and upper bound functions are defined as in Eq.~\ref{func:globbound}, where $\globalLower{d}{:}$ and $\globalUpper{d}{:}$ denotes the lower and the upper bound for the embedding of the $d$-th input words vector separately. In this formulation, $\ablinearw{0}{l}$ represents the cumulative weight after abstraction from $0$-th to $l$-th layer associated with this particular bound, and $\ablinearb{0}{l}{d}$ is the bias term. Here, $\highXEle{d,:}$ denotes the $d$-th original embedding vector from the clean (unperturbed) input, while $\epsilon$ is the radius of perturbation. The term $\norm[q]{\cdot}$ denotes the dual norm with respect to the perturbation norm $p$, satisfying the relation $1/p + 1/q = 1$.

\begin{equation}
	\begin{aligned}
		\globalLower{d}{:} = -\epsilon \norm[q]{\ablinearw{0}{l}} + \ablinearw{0}{l}\highXEle{d,:} + \ablinearb{0}{l}{d},~~~~
		\globalUpper{d}{:} = \epsilon \norm[q]{\ablinearw{0}{l}} + \ablinearw{0}{l}\highXEle{d,:} + \ablinearb{0}{l}{d}
		\label{func:globbound}
	\end{aligned}
\end{equation}

\subsection{Transformation for Linear Mapping}
Consider a linear affine mapping in a neural network for the $d$-th neuron vector $\neuronvec{l}{d}$ connecting the last neuron vector $\neuronvec{l-1}{d}$, from $(l-1)$-th layer to the $l$-th layer. Formally, the affine layer of the neural network is defined as:
\begin{equation}
	\neuronvec{l}{d} = \neuronw{l-1}\neuronvec{l-1}{d}+\neuronb{l-1}{d}
	\label{func:linearfunc}
\end{equation}
Here, $\neuronw{l-1}$ and $\neuronb{l-1}{d}$ denote the weight matrix and bias vector of the trained neural network under consideration, respectively.
We now switch from the network forward computation to the affine abstraction used for verification.
To propagate an existing abstracted linear transformation Eq.~\ref{func:lineartrans} backward from $(l+1)$-th to $(l)$-th layer, through Eq.~\ref{func:linearfunc}, we substitute the $\neuronvec{l}{i}$ into the transformation expression.
\begin{equation}
	\neuronvec{l+1}{d} = \ablinearw{l}{l+1}\neuronvec{l}{d} + \ablinearb{l}{l+1}{d}
	\label{func:lineartrans}
\end{equation}
Specifically, the abstract transformation can be propagated in a backward method from layer $(l+1)$ to layer $(l-1)$ as follows: 
\begin{equation}
	\begin{aligned}
		\neuronvec{l+1}{d} &=\ablinearw{l-1}{l+1}\neuronvec{l-1}{d} + \ablinearb{l-1}{l+1}{d} \\
		\ablinearw{l-1}{l+1} &=\ablinearw{l}{l+1}\neuronw{l-1} \\
            \ablinearb{l-1}{l+1}{d} &= \ablinearw{l}{l+1}\neuronb{l-1}{d} + \ablinearb{l}{l+1}{d}
	\end{aligned}
	\label{func:linearprop}
\end{equation}
Eq.~\ref{func:linearprop} shows how the transformation coefficients \ablinearw{l}{l+1} and \ablinearb{l}{l+1}{d} from $(l+1)$-th to $(l)$-th layer are propagated backward to $(l-1)$-th layer using the network parameters \neuronw{\cdot} and \neuronb{\cdot}{\cdot}.

\smallskip\noindent\textit{Model parameters vs. abstract coefficients.}
We distinguish two types of coefficients throughout this appendix.
The symbols \neuronw{\cdot} and \neuronb{\cdot}{\cdot} denote the weight matrices and bias vectors of the neural network, which are fixed after training.
In contrast, the symbols \ablinearw{\cdot}{\cdot} and \ablinearb{\cdot}{\cdot}{\cdot} denote affine coefficients introduced by the verification procedure.
They are not model parameters, but abstract quantities.
These transformation coefficients are iteratively propagated backward through layers using the network parameters.

\subsection{Relaxation for Unary Nonlinear Function}
In the case of unary nonlinear operations, such as the ReLU function $\sigma(\cdot)$, the output of the $l$-th layer is obtained by applying a elemental-wise nonlinearity to the $(l-1)$-th layer:
\begin{equation}
	\neuronvec{l}{i} = \sigma(\neuronvec{l-1}{i})
	\label{func:6}
\end{equation}
To enable linear bound propagation through such nonlinear activation function, we relax the $\sigma(\cdot)$ using a pair of linear functions that tightly bound it from below and above.
Specifically, for each neuron indexed by $j$, we define a linear lower and upper relaxation:
\begin{equation}
	\begin{aligned}
		\relualLower{l-1}{l}{i}{j}\neuron{l-1}{i}{j}+\relubeLower{l-1}{l}{i}{j}\le\sigma(\neuron{l-1}{i}{j})\le\relualUpper{l-1}{l}{i}{j}\neuron{l-1}{i}{j}+\relubeUpper{l-1}{l}{i}{j}
	\end{aligned}
	\label{func:7}
\end{equation}
Here, $\relualLower{l-1}{l}{i}{j}$ and $\relualUpper{l-1}{l}{i}{j}$ are the slopes of the linear abstract transformation of lower and upper bounds, respectively, while $\relubeLower{l-1}{l}{i}{j}$ and $\relubeUpper{l-1}{l}{i}{j}$ are the corresponding intercepts.
These parameters are typically determined based on the concrete lower and upper bounds of the input to the nonlinearity, and they ensure that the relaxation tightly encloses the true nonlinear activation within a convex region.

Specially, for an upper bound, $\relualUpper{l-1}{l}{i}{j}$ is set to the slope of the line connecting maximum and minimum point in cross-zero situation. $\relubeUpper{l-1}{l}{i}{j}$ is the corresponding intercept. There is a judgment when considering lower bounds. If the negative domain is greater than the positive domain, set the slope to 0; otherwise, set it to 1. The intercept is always 0 \cite{zhang2018efficient}.

\subsection{Detailed Introduction to Abstraction for Dot-Product Operation and Softmax Function}
\cite{Shi2020Robustness} addressed the verification of dot-product operation such as $z = xy$ by introducing planar relaxations of the form $z = \alpha x + \beta y + \gamma$ to bound the nonlinearity.
In particular, \cite{Shi2020Robustness} define the lower and upper bounds as follows: $\alpha^U=u_y,~\beta^U=l_x,~\gamma^U=-l_xu_y$, $\alpha^L=l_y,~\beta^L=l_x,~\gamma^L=-l_xl_y$, when $x \in [l_x, u_x]$ and $y \in [l_y, u_y]$.

For the dot-product operation $QK^T$, where $Q = \qkvweight{Q}\highX + \qkvbias{Q}$ and $K = \qkvweight{K}\highX + \qkvbias{K}$, the dot-product between queries and keys involves two variable terms that are each affine transformations of the input $X$. \cite{Shi2020Robustness} further transform the bilinear relaxation planes into single-variable forms. This is achieved by substituting the expressions for $Q$ and $K$ into the planar bounds, thereby obtaining bounds that depend only on the input $X$, while preserving the affine structure of the relaxation.

A detailed derivation of this substitution process can be found in Section 3.3 of \cite{Shi2020Robustness}, and the process is summarized like in Eq.~\ref{func:qk}.

Given $Q\in \R^{n\times m}$ and $K\in \R^{n\times m}$, $Q\transpose{K}$ results in an $n\times n$ matrix, and each perturbed element can be represented as
\begin{math}
\sum_{j=1}^{m}q_{ij}k_{ji}
\end{math}, where $i\in \{1,\ldots, n\}$.

For each $q_{ij}k_{ji}$, we have
\begin{equation}
    \begin{aligned}
        &\alphaL q_{ij}+\betaL k_{ji}+\gammaL \le q_{ij}k_{ji} \le \alphaU q_{ij}+\betaU k_{ji}+\gammaU\\
        &=\alpha_i(\qkvweightEle{Q}{:}{j}\highXEle{i,:} + \qkvbiasEle{Q}{i}{j})+\beta_i(\qkvweightEle{K}{:}{j}\highXEle{i,:} + \qkvbiasEle{K}{i}{j})+\gamma_i\\
        &=(\alpha_i\qkvweightEle{Q}{:}{j}+\beta_i\qkvweightEle{K}{:}{j})\highXEle{i,:}+(\alpha_i\qkvbiasEle{Q}{i}{j}+\beta_i\qkvbiasEle{K}{i}{j}+\gamma_i)
\end{aligned}
\label{func:qk}
\end{equation}

For the softmax function, just like~\cite{Shi2020Robustness}, we first abstract the exponential function as a linear transformation. Then, we decompose the softmax into a numerator and a reciprocal-form denominator, allowing the reciprocal function to be further abstracted. Ultimately, the formulation can be reduced to the dot-product problem introduced earlier.

\section{Proof for Lemma~\ref{lem:dualBounds}}\label{sec:lemmaProof}
\begin{proof}
We first show that the difference between the upper bound (as defined in Eq.~\ref{eq:dualBounds}) and the bilinear term $\qEle\kEle$ is not possible to be negative. The difference can be computed as follows:
\begin{align*}
\qkUpperTwo{\qEle, \kEle} - \qEle\kEle &= \kLower\qEle + \qUpper  \kEle - \qUpper \kLower - \qEle \kEle \\
&= \qEle(\kLower -\kEle) + \qUpper(\kEle - \kLower) \\
&= (\kLower - \kEle)(\qEle - \qUpper).
\end{align*}

Since $\qEle \in [\qLower, \qUpper]$ and $\kEle \in [\kLower, \kUpper]$, we have
\begin{align*}
\qEle - \qUpper \le 0, \quad \kLower - \kEle \le 0.
\end{align*}
Therefore, the product of these two non-positive terms is always non-negative:
\begin{align*}
(\kLower - \kEle)(\qEle - \qUpper) \ge 0.
\end{align*}

Consequently, we can derive that,
\begin{math}
\qkUpperTwo{\qEle, \kEle} - \qEle \kEle \ge 0,
\end{math},
i.e.,
\begin{math}
\qEle \cdot \kEle \le \qkUpperTwo{\qEle, \kEle}.
\end{math}

Similarly, we can prove that the difference between the lower bound (as defined in Eq.~\ref{eq:dualBounds}) and the bilinear term $\qEle\kEle$ is not possible to be positive. Therefore, Lemma~\ref{lem:dualBounds} holds.
\end{proof}

\section{Equivalence between Eq.~\ref{eq:jointU} and Eq.~\ref{eq:reluU}}
\label{sec:equal}
We show that the ReLU-based formulations for the combined bounds are mathematically equivalent to the min/max-based definitions.

Let $a := \qkUpperOne{\qEle,\kEle} $ and $b := \qkUpperTwo{\qEle,\kEle}$. The ReLU-based upper bound is defined as:
\begin{equation}
\qkUpperNonlin{\qEle,\kEle} := a - \text{ReLU}(a - b)
\end{equation}

We consider two cases:

\begin{itemize}
    \item If $a \leq b$, then $\text{ReLU}(a - b) = 0$, so $\qkUpperNonlin{\qEle,\kEle} = a =\min(a, b)$.
    \item If $a > b$, then $\text{ReLU}(a - b) = a - b$, so $\qkUpperNonlin{\qEle,\kEle} = a - (a - b) = b = \min(a, b)$.
\end{itemize}

Therefore, we have:
\begin{equation}
\qkUpperNonlin{\qEle,\kEle} = \min(\qkUpperOne{\qEle,\kEle}, \qkUpperTwo{\qEle,\kEle})
\end{equation}

Similarly, for the lower bound, let $a := \qkLowerOne{\qEle,\kEle}$ and $b := \qkLowerTwo{\qEle,\kEle}$. The ReLU-based formulation is:
\begin{equation}
\qkLowerNonlin{\qEle,\kEle} := a + \text{ReLU}(b - a)
\end{equation}

Again, we consider two cases:
\begin{itemize}
    \item If $a \geq b$, then $\text{ReLU}(b - a) = 0$, so $\qkLowerNonlin{\qEle,\kEle} = a = \max(a, b)$.
    \item If $a < b$, then $\text{ReLU}(b - a) = b - a$, so $\qkLowerNonlin{\qEle,\kEle} = a + (b - a) = b = \max(a, b)$.
\end{itemize}

Thus, we conclude:
\begin{equation}
\qkLowerNonlin{\qEle,\kEle} = \max(\qkLowerOne{\qEle,\kEle}, \qkLowerTwo{\qEle,\kEle})
\end{equation}

Hence, the ReLU-based formulations are fully equivalent to the traditional min/max formulations.

\begin{figure}[!tb]
\vspace{4em}
  \centering
  \begin{tabular}{@{}c c c@{}}

    \begin{minipage}[t]{0.50\textwidth}
    \centering
      \small
      \setlength{\tabcolsep}{4pt}
      \renewcommand{\arraystretch}{1.05}
		\begin{tabular}{lccc}
			\toprule
			Dataset & Task      & Train   & Test   \\
			\midrule
            			\yelp    & Polarity  & 560,000 & 38,000 \\
			\sst     & Sentiment & 67,349  & 1,821  \\[2pt]

			\bottomrule
		\end{tabular}
        \captionof{table}{Dataset statistics.}
        \label{table:datasets}
    \end{minipage}


    \multirow[t]{2}{*}{%
      \begin{minipage}[t]{0.45\textwidth}
        \centering
        \setlength{\tabcolsep}{4pt}
        \renewcommand{\arraystretch}{1.2}
        \begin{tabular}{crr}
				\toprule
				 \multirow{2}*{\modelLayer} &  \multicolumn{2}{c}{accuracy}\\
                \cline{2-3}
                                  ~&  \sst  &    \yelp \\
                
				\midrule
                1 & 0.826                           & 0.912                       \\
				2 & 0.832                           &  0.912                     \\
				3 & 0.823                           & 0.913                      \\ 
                6 & 0.831                           & 0.914                    \\ 
                \cline{1-3}
                {\scriptsize\textit{TinyBERT}} & 0.809            & 0.915        \\ 
				\bottomrule
			\end{tabular}
        \captionof{table}{The accuracy of two benchmarks.}
        \label{table:accuracy}
      \end{minipage}
    }


  \end{tabular}
  \vspace{2em}
\end{figure}

\section{Extended Experiment Settings}\label{sec:extendedExperimentSettings}
We give more details about the datasets and model configurations used in our experiments. Table~\ref{table:datasets} presents statistics for two benchmark datasets: 
\begin{itemize}
\item \textbf{Yelp Polarity}~\cite{socher2013recursive}:
A large-scale binary sentiment classification dataset constructed from Yelp reviews.
\begin{itemize}
\item Each review is labeled as either \textit{positive} or \textit{negative}, based on the original star rating.
\item Contains \textbf{560,000} training samples and \textbf{38,000} test samples.
\item Designed to evaluate robust sentiment classification at scale.
\end{itemize}

\item \textbf{SST (Stanford Sentiment Treebank)}~\cite{zhang2015character}:
A widely-used benchmark dataset based on movie reviews from Rotten Tomatoes.
\begin{itemize}
\item We use the \textit{binary version} (SST), which labels each sentence as either \textit{positive} or \textit{negative}.
\item Contains \textbf{67,349} training samples, \textbf{872} validation samples, and \textbf{1,821} test samples.
\item Known for its phrase-level annotations and grammatical structure, making it suitable for fine-grained sentiment analysis.
\end{itemize}
\end{itemize}

Each dataset is associated with a sentiment classification task, and we report the number of training and test instances.

Table~\ref{table:accuracy} details the accuracy of the models we adopted in our experiments. These accuracy data show that the models we adopted in our experiments are well-trained models with reasonable accuracy, so they can be used for classification tasks. However, their robustness levels are not known beforehand, so we apply our approaches to pursue the maximal $\epsilon$ for which the models are still robust.


\begin{algorithm}[!tb]
\caption{Binary Search for Safety Verification}
\label{alg:ar1}
\KwIn{$\highX$, $num\_iters$}
\KwOut{The maximum safe radius $\epsilon$}
$min\_eps \gets 0$\;
$max\_eps \gets 0.01$\;
\While{$\texttt{verify\_safety}($\highX$, max\_eps)$}{
$min\_eps \gets max\_eps$\;
$max\_eps \gets max\_eps \times 2$\;
}
\For{$i \gets 1$ \KwTo $num\_iters$}{
    $m \gets (min\_eps + max\_eps)/2$\;
    \If{$\texttt{verify\_safety}($\highX$, m)$}{
            $min\_eps \gets m$\;
        
        }
        \Else{
            $max\_eps \gets m$\;
        }
}
$\epsilon \gets min\_eps$\;
\Return $\epsilon$\;
\end{algorithm}

\section{Binary Search}\label{sec:binarySearch}
As mentioned in Section~\ref{sec:problemStatement}, an important application of verification is to find the maximum certified safe radius $\epsilon$, under perturbing a given input $\highX$ is provably safe. There could be different ways to search for such $\epsilon$. In our experiments, we adopt a natural and common method following~\cite{zhang2018efficient,Shi2020Robustness}, which involves a binary search method.

Algorithm~\ref{alg:ar1} outlines the procedure for searching such $\epsilon$,  under a fixed verification method. The algorithm begins by initializing the minimum and maximum of $\epsilon$ as $0$ and $0.01$, respectively. It then exponentially increases the upper bound until the safety condition $\texttt{verify\_safety}(\highX, \epsilon)$ fails, thereby identifying an interval $[\texttt{min\_eps}, \texttt{max\_eps}]$ that contains the boundary of safety.

Subsequently, a binary search is conducted within this interval for a fixed number of iterations (num\_iters) to refine the estimate of the maximum safe $\epsilon$. At each step, the midpoint $m$ is checked using the verification oracle. If $m$ satisfies the safety property, it becomes the new minimum value; otherwise, it becomes the maximum value. After convergence, the algorithm returns $\epsilon := \texttt{min\_eps}$ as the largest radius for which the model can be formally verified to behave correctly.

\section{Horizontal Comparison with LSE~\cite{pmlr-v206-wei23c}}
\label{sec:horizontal-lse}

We additionally conduct a small-scale horizontal comparison with the LSE
softmax relaxation proposed by Wei et al.~\cite{pmlr-v206-wei23c}. The goal of
this experiment is not to position \toola and LSE as mutually exclusive
alternatives. Instead, the two techniques refine different sources of
approximation error in transformer verification. LSE directly tightens the
relaxation of the softmax operator, while \toola refines the linear relaxation
of attention dot products by fusing alternative planar bounds through
ReLU-based abstractions. Therefore, the two methods can be naturally combined
within the same bound-propagation pipeline.

To illustrate this complementarity, we evaluate a fixed-radius verification
task on the 3-layer SST model. Unlike the main certified-radius experiments,
here we fix the input perturbation radius and compare the resulting lower bound
on the classification margin. A larger margin lower bound indicates a tighter
verification result for the same robustness task. We use the same examples,
token positions, perturbation norm, and radius for all methods. To compare the
methods on exactly matched verification tasks, we report the average normalized
margin ratio
\[
    \frac{1}{n}\sum_{i=1}^{n}
    \left(1 + \frac{m_i - m_i^{\mathrm{base}}}{|m_i^{\mathrm{base}}|}\right),
\]
where \(m_i\) is the \marginLow defined in Eq~\ref{eq:margin} of a method on task \(i\), and
\(m_i^{\mathrm{base}}\) is the corresponding baseline margin. Under this
normalization, the baseline ratio is \(1.0\), and larger values indicate tighter
margins on the same tasks. For o-BuFFeT and the combined method, we optimize
the ReLU relaxation parameters with Adam using zero initialization, learning
rate 0.05, and 100 optimization steps.

\begin{table}[!tb]
\centering
\caption{Horizontal comparison with LSE on the 3-layer SST model at fixed
perturbation radius. We report the average normalized margin ratio over matched
verification tasks; higher is better. The LSE relaxation and \toola refine
different components of attention verification, and their combination achieves
the largest relative margin improvement.}
\label{table:horizontal-lse}
\setlength{\tabcolsep}{5pt}
\resizebox{\linewidth}{!}{
      \begin{tabular}{lcc}
            \toprule
            method & average margin ratio & relative gain \\
            \midrule
            baseline & 1.00000 & 0.00\% \\
            wei\_lse & 1.01182 & 1.182\% \\
            \toola & 1.00463 & 0.463\% \\
            wei\_lse+\toola & 1.01610 & 1.610\% \\
            \bottomrule
        \end{tabular}
}
\end{table}

As shown in Table~\ref{table:horizontal-lse}, the LSE relaxation improves the
normalized margin ratio over the baseline, and \toola also improves the
ratio when its optimization parameters are tuned for this fixed-radius task.
More importantly, combining LSE with \toola further improves the ratio beyond
either method alone. This supports our view that LSE-style softmax relaxations
and \toola are complementary: one targets the softmax component, while the other
targets the attention dot-product component. Therefore, users can consider the two lines of relaxation as complementary tools, and that combining softmax-specific relaxations with dot-product-specific refinement can provide a strictly stronger verification pipeline in practice.


\end{document}